\documentclass{article}

\usepackage{microtype}
\usepackage{graphicx}
\usepackage{subcaption}
\usepackage{booktabs}
\usepackage{subcaption}

\usepackage{hyperref}

\usepackage[accepted]{icml2026} 
\usepackage{amsmath}
\usepackage{amssymb}
\usepackage{mathtools}
\usepackage{amsthm}

\usepackage{tikz}
\usetikzlibrary{positioning,calc,arrows.meta}

\usepackage[textsize=tiny]{todonotes}

\usepackage{amsfonts,bm}

\def\eqref#1{equation~\ref{#1}}

\def\1{\bm{1}}

\DeclareMathAlphabet{\mathsfit}{\encodingdefault}{\sfdefault}{m}{sl}
\SetMathAlphabet{\mathsfit}{bold}{\encodingdefault}{\sfdefault}{bx}{n}

\theoremstyle{plain}
\newtheorem{theorem}{Theorem}[section]
\newtheorem{proposition}[theorem]{Proposition}
\newtheorem{lemma}[theorem]{Lemma}
\newtheorem{corollary}[theorem]{Corollary}
\theoremstyle{definition}

\theoremstyle{remark}

\usepackage{enumitem}

\usepackage{thm-restate}
\providecommand{\definitionname}{Definition}
\providecommand{\lemmaname}{Lemma}
\providecommand{\theoremname}{Theorem}
\providecommand{\definitionname}{Definition}
\providecommand{\lemmaname}{Lemma}
\providecommand{\theoremname}{Theorem}

\usepackage[utf8]{inputenc} 
\usepackage[T1]{fontenc}    
\usepackage{hyperref}       
\usepackage[capitalize,noabbrev]{cleveref}
\usepackage{url}            
\usepackage{booktabs}       
\usepackage{amsfonts}       
\usepackage{nicefrac}       
\usepackage{microtype}      
\usepackage{xcolor}         
\usepackage{titletoc}

\crefname{equation}{Eq.}{Eqs.}
\Crefname{equation}{Equation}{Equations}
\newcommand{\bS}{\boldsymbol{S}}
\newcommand{\bmu}{\boldsymbol{\mu}}
\newcommand{\bb}{\boldsymbol{b}}

\newcommand{\bx}{\boldsymbol{x}}
\newcommand{\bz}{\boldsymbol{z}}
\newcommand{\by}{\boldsymbol{y}}
\newcommand{\avg}[1]{\mathbb{E}\left[ #1 \right]}

\crefname{section}{Sec.}{Secs.}
\Crefname{section}{Section}{Sections}

\crefname{appendix}{App.}{Apps.}
\Crefname{appendix}{Appendix}{Appendixes}

\crefname{theorem}{Theorem}{Theorems}
\crefname{lemma}{Lemma}{Lemmas}
\crefname{proposition}{Proposition}{Propositions}
\crefname{corollary}{Corollary}{Corollaries}

\Crefname{theorem}{Theorem}{Theorems}
\Crefname{lemma}{Lemma}{Lemmas}
\Crefname{proposition}{Proposition}{Propositions}
\Crefname{corollary}{Corollary}{Corollaries}

\hypersetup{
    colorlinks,
    linkcolor={blue!50!black},
    citecolor={blue!50!black},
    urlcolor={blue!80!black}
}

\usepackage{thm-restate}

\providecommand{\definitionname}{Definition}
\providecommand{\lemmaname}{Lemma}
\providecommand{\theoremname}{Theorem}
\providecommand{\definitionname}{Definition}
\providecommand{\lemmaname}{Lemma}
\providecommand{\theoremname}{Theorem}

\newcommand{\norm}[1]{\left\lVert{#1}\right\rVert}

\newcommand{\bs}{\boldsymbol{S}}

\icmltitlerunning{}

\begin{document}
\pagestyle{plain} 
\twocolumn[
\icmltitle{The Implicit Bias of Logit Regularization}

\begin{icmlauthorlist}
\icmlauthor{Alon~Beck}{tau}
\icmlauthor{Yohai~Bar-Sinai}{tau,racah,hujics}
\icmlauthor{Noam~Levi}{epfl}
\end{icmlauthorlist}

\icmlaffiliation{tau}{Raymond and Beverly Sackler School of Physics and Astronomy, Tel-Aviv University, Tel-Aviv 69978, Israel}
\icmlaffiliation{epfl}{École Polytechnique Fédérale de Lausanne (EPFL), Switzerland}
\icmlaffiliation{racah}{Racah Institute of Physics, The Hebrew University of Jerusalem, Edmond J. Safra Campus, Jerusalem 9190401, Israel}
\icmlaffiliation{hujics}{The Rachel and Selim Benin School of Computer Science and Engineering, The Hebrew University of Jerusalem, Edmond J. Safra Campus, Jerusalem 9190401, Israel}

\icmlcorrespondingauthor{Alon~Beck}{alonbk2@gmail.com}
\icmlkeywords{label smoothing, logit regularization, implicit bias, high-dimensional asymptotics, grokking}

\vskip 0.3in
]

\printAffiliationsAndNotice{}  


\begin{abstract}
Logit regularization, the addition of a convex penalty directly in logit space, is widely used in modern classifiers, with label smoothing as a prominent example. While such methods often improve calibration and generalization, their mechanism remains under-explored.
In this work, we analyze a general class of such logit regularizers in the context of linear classification, and demonstrate that they induce an implicit bias of \emph{logit clustering} around finite per-sample targets.
For Gaussian data, or whenever logits are sufficiently clustered, we prove that logit clustering drives the weight vector to align exactly with \emph{Fisher's Linear Discriminant}.
To demonstrate the consequences, we study a simple signal-plus-noise model in which this transition has dramatic effects:
Logit regularization halves the critical sample complexity and induces grokking in the small-noise limit, while making generalization robust to noise.
Our results extend the theoretical understanding of label smoothing and highlight the efficacy of a broader class of logit-regularization methods.
\end{abstract}

\section{Introduction}

Modern deep learning models are frequently trained with Label Smoothing (LS) \citep{szegedy2016rethinking} or similar regularization techniques to improve generalization and calibration \citep{muller2019when}. The empirical benefits of these methods are well-documented, and their success is often attributed to the prevention of overconfidence. However, a theoretical understanding of their underlying mechanics remains incomplete. 

In this work, we investigate the geometric mechanism of a broader class of regularizers: \emph{convex logit regularization}, of which Label Smoothing is one instance. Our central observation is that LS causes clustering of logits. 

Unlike weight-based regularization schemes, such as $L_2$, or ridge regularization~\cite{Tikhonov1977SolutionsOI}, LR allows a per-sample decomposition of the loss as $\mathcal{L}=N^{-1}\sum_i \ell(\bz_i,\by_i)$, where $\bz_i$ and $\by_i$ denote the logit and class of the $i$-th sample, respectively, and $N$ is the number of samples. A convex logit penalty generically creates a finite global minimum for $\ell$, and thus gradient descent performs \emph{logit clustering} around these finite targets in logit space. This is in contrast to unregularized logisitic regression, in which gradient descent is equivalent to margin maximization~\cite{soudry2018implicit,ji2019implicit,lyu2020margin}.
Leveraging this insight, we derive several surprising consequences of logit-regularization. We summarize our main contributions below:

\begin{itemize}
    \item \textbf{The implicit bias of Logit regularization:} We demonstrate that logit regularization fundamentally shifts the optimization objective from margin maximization to logit clustering around finite per-sample targets.
    For Gaussian data or quadratic per-sample loss, we prove that this clustering objective drives the weight vector $\bs$ to align exactly with \emph{Fisher's Linear Discriminant} \citep{fisher1936use}, $\bs \propto \Sigma^{-1} \boldsymbol{\mu}$.
    
    \item \textbf{Insensitivity to the regularizer:}
    Beyond the solvable Gaussian/quadratic cases, we argue the same direction holds approximately whenever the data is near-Gaussian or when logits concentrate near their targets (making the loss locally quadratic). As this direction depends only on data geometry, generalization accuracy becomes largely insensitive to the regularizer’s functional form and strength.

    \item We assume a signal-plus-noise decomposition of the data, that splits the noisy part of the input features into a signal-aligned component and an orthogonal complement.
    In this setting, we demonstrate the following.
    \begin{itemize}
        \item \textbf{Shift in the Interpolation Threshold and grokking:}
         In the limit of noiseless features (where the noise component along the signal direction is zero) and working in the high-dimensional proportional regime of $d,N\to\infty$ (where $d$ is the input dimension), logit regularization shifts the interpolation threshold from the standard separability limit $\lambda_c = d/N = 1/2$~\citep{cover1965geometrical,gardner1988space, beck2025grokking} to $\lambda_c = 1$. Notably, under weak regularization, we identify \emph{grokking} dynamics (delayed generalization accompanied by non-monotonic test loss dynamics) in the region $1/2 < \lambda < 1$.
        
        \item \textbf{Invariance to Orthogonal Noise:} We prove that the optimal generalization accuracy is invariant to the amplitude of orthogonal noise. This observation allows us to construct a \textit{phase diagram} delineating the specific regimes (defined by signal noise and orthogonal noise amplitudes) where logit regularization yields a generalization benefit over the unregularized baseline.
        
    \end{itemize}

\end{itemize}

We validate our theoretical predictions numerically on Gaussian data and on a more realistic distribution using ResNet-18~\cite{resnet18} embeddings from the CIFAR-10 dataset~\citep{krizhevsky2009cifar}.

\section{Related Work}

Here, we provide a focused review of the prior research most central to our contribution. For an
additional detailed literature survey, including variants of LS and technical high-dimensional solvers, we refer the reader to~\cref{appendix:related_work}.

\textbf{Logit Regularization.} Label smoothing (LS) \citep{szegedy2016rethinking} is a standard tool for improving generalization and calibration \citep{muller2019when,guo2017calibration}. It is a specific instance of a broader class of logit-space penalties, including entropy-based \citep{pereyra2017regularizing} and logit-norm \citep{dauphin2021deconstructing} regularizers. While these are often motivated by the prevention of overconfidence, our work provides a geometric characterization of how these penalties replace unbounded margin growth with a \emph{logit clustering} objective.

\textbf{Implicit Bias and High-Dimensional Asymptotics.} For unregularized classification, the implicit bias of gradient descent leads to the hard-margin SVM/max-margin separator \citep{soudry2018implicit,lyu2020margin}. We contrast this with the finite per-sample optima created by logit regularization. Our analysis connects to high-dimensional theories of linear separability \citep{cover1965geometrical,gardner1988space} and recent asymptotics for overparameterized classifiers \citep{montanari2019maxmargin,mignacco2020regularization,wang2022gaussianmixtures}.

\textbf{Grokking and Neural Collapse.} Grokking, a delayed transition to generalization, was originally observed in algorithmic tasks \citep{power2022grokking}. The most relevant works in this field relate the edge of separability in logistic regression with grokking transitions~\citep{beck2025grokking}. In multi-class settings, are results are reminiscent of \emph{neural collapse} \citep{papyan2020neuralcollapse}, a phenomenon in which logits arrange in simplex configurations \citep{zhou2022all,garrod2024persistence}.

\section{The Implicit Bias of Logit Regularization}

\subsection{Setup: Logit Regularization as a Generalization of Label Smoothing}
We study linear classifiers. For conceptual clarity, we focus initially on the binary setting with zero bias ($b=0$). In this case, the model outputs a scalar logit $z_i = \boldsymbol{S}^\top \boldsymbol{x}_i$ where $\boldsymbol{x}_i$ are the data points, and $\boldsymbol{S}$ is model weights vector. The extension to the multi-class setting is provided in \cref{sec:multiclass}.

We consider a class of loss functions $\mathcal{L}=N^{-1}\sum_{i=1}^N\ell(z_i,y_i)$, where $z_i$ is the logit and $y_i \in \{-1, +1\}$  is label of the $i$-th sample, respectively. We consider the per-sample loss
\begin{equation}
\label{eq:single_sample_loss}
    \ell(z, y) = (1-\alpha)\ell_{\mathrm{CE}}(z,y) + \alpha f(z),
\end{equation}
where $\ell_{\mathrm{CE}}(z,y)=\log(1+e^{-yz})$ is the standard cross-entropy loss, $f(z)$ is a convex and even \emph{logit regularization} function, and $\alpha \in [0, 1)$ controls the regularization amplitude.
Notably, standard Label Smoothing (LS) is  a specific instance of this framework, corresponding to the choice $f(z) = \log(2+2\cosh z)$ (See~\cref{appendix:label_smoothing}). While most of our results do not depend on the exact form $f$, for simplicity, we will consider $f(z)=z^2$ in all our numerical results. 

Since $\ell_{\mathrm{CE}}$ depends only on the product $yz$ and $f$ is even, the loss depends strictly on the signed logit. For notational simplicity, from here on we redefine $z \leftarrow yz$. Equivalently, we absorb the label into the input vector, $\boldsymbol{x} \leftarrow y\boldsymbol{x}$, such that the logit remains $z_i = \boldsymbol{S}^\top \boldsymbol{x}_i$. This allows us to analyze the univariate function $\ell(z)$ independently of the class label.

\emph{Distinction from $L_2$ regularization:} We emphasize that this formulation differs significantly from standard parameter regularization. While logit regularization acts on the model logits independently, $L_2$ penalty implicitly couples data samples via the global norm. This prevents decomposing the loss into independent per-sample terms, a property central to all of our analysis throughout this work.

\subsection{Logit Clustering}
\label{sec:clustering}
The introduction of logit regularization fundamentally alters the optimization landscape. In the unregularized case ($\alpha=0$), the per-sample cross-entropy loss $\ell_
{\mathrm{CE}}$ is monotonic, and the optimal solution maximizes the margins \citep{soudry2018implicit,ji2019implicit}.
When the data is linearly separable, this implicit bias drives the weight norm $\|\boldsymbol{S}\| \to \infty$, potentially leads to overconfidence and overfitting (see \cref{fig:logit_clustering}, top row).
In contrast, when $\alpha > 0$, the per-sample loss $\ell(z)$ has a unique finite global minimum denoted by $z^*$. Clearly, the loss has a data-independent lower bound, $\mathcal{L}\ge z^* $, which is achieved only if $z_i=z^*$ for all samples.
However, given the linear constraint $z_i = \boldsymbol{S}^\top \boldsymbol{x}_i$, the model generally cannot collapse all samples to a single point. Consequently, the optimization process becomes a geometric compromise: the model seeks a weight vector $\boldsymbol{S}$ that \textbf{clusters all logits as tightly as possible around the target $z^*$.} This phenomenon is visualized in the bottom row of \cref{fig:logit_clustering}.
While logit clustering is a known effect of label smoothing \citep{muller2019when}, we will show that this straightforward observation has important consequenses.

\begin{figure}[t!]
    \centering
    \includegraphics[scale=0.32]{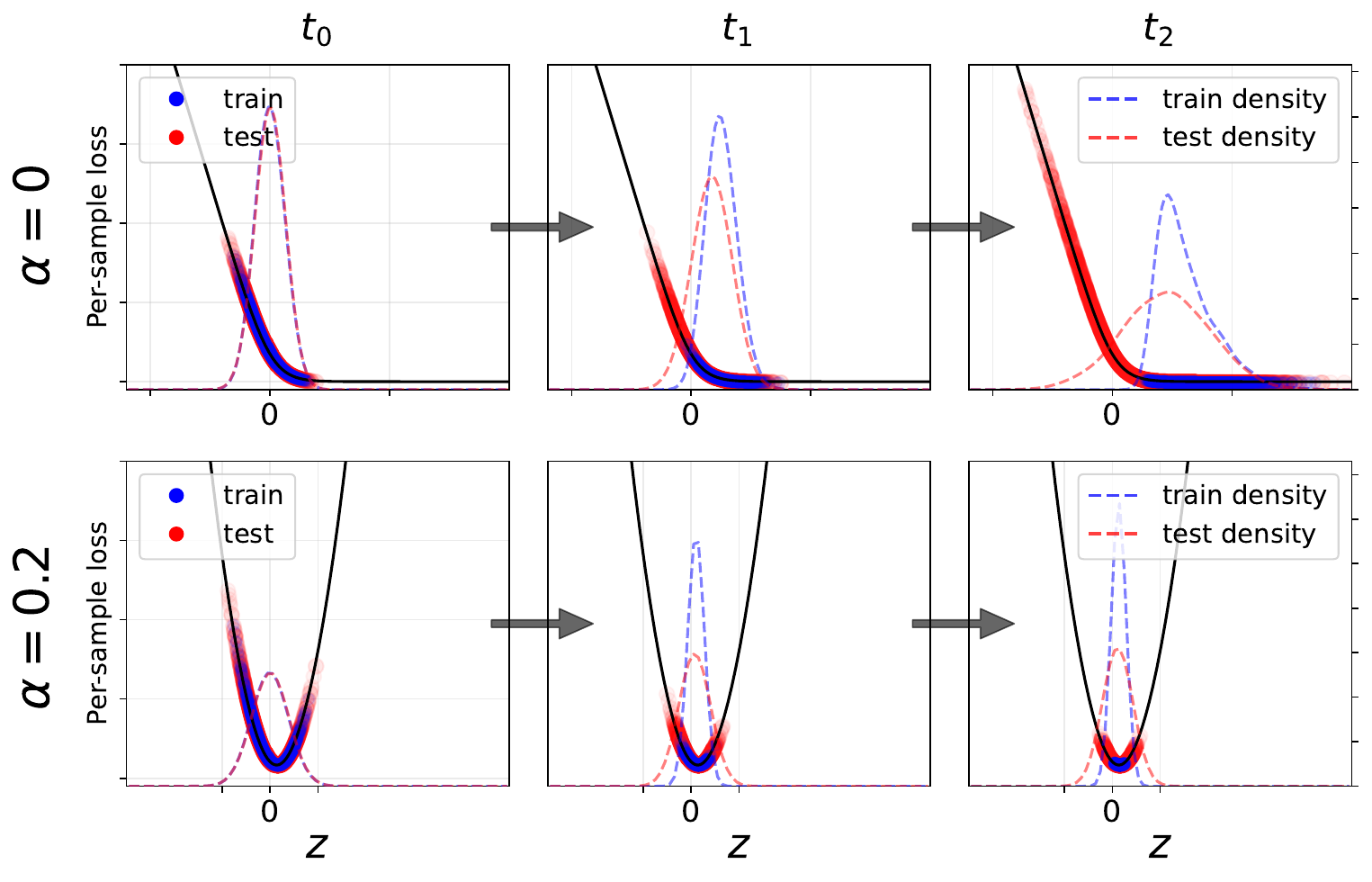}
    \caption{\textbf{Logit regularization induces clustering.} 
    Logit evolution for a linear classifier trained on Gaussian data (see \cref{app:numerical_details} for more details regarding the setup). The train (blue) and test (red) samples are visualized across three training epochs (early, middle, and final). Note that samples are classified correctly if $z>0$. \textbf{Top ($\alpha=0$):} The unregularized logits are pushed toward infinity to maximize margins, indicating overconfidence. \textbf{Bottom ($\alpha > 0$):} The regularized loss exhibits a distinct finite minimum, driving the logits to cluster tightly around a target value $z^*$ during training. Note the looser clustering on the test set due to the misalignment between the empirical noise direction and the true signal.
    }
    \label{fig:logit_clustering}
\end{figure}

\subsection{Alignment with Fisher's Linear Discriminant}
In this section, we derive the exact form of the implicit bias of logit regularization for logistic regression under idealized conditions, and show that logit clustering in this case drives the weight vector to be in the direction of Fisher's Linear Discriminant. Specifically, we consider two regimes: (i) Gaussian data with any convex regularizer, and (ii) arbitrary data distributions with a quadratic per-sample loss\footnote{Almost any per-sample loss as defined in~\cref{eq:single_sample_loss} can be considered quadratic close enough to its minimum.}. In both cases, we  first demonstrate that the optimization problem is equivalent to minimizing the \emph{coefficient of variation} (CV) $r(\bS)\equiv\sigma(\bS)/\mu(\bS)$, where $\mu(\bs),\sigma(\bs)$ are the logits mean and standard deviation, as a function of the weights.

\begin{proposition}[Gaussian Data]\label[proposition]{prop:gaussian_data}
    Let $\bx \sim \mathcal{N}(\boldsymbol{\mu}_{x}, \Sigma_{x})$ with $\boldsymbol{\mu}_{x} \neq \mathbf{0} \in \mathbb{R}^d$, and consider a vector $\bs\in\mathbb{R}^d$.
    The logit $z(\bs) = \bs^{T}\bx$ is Gaussian with mean
    $\mu(\bs) = \bs^{T}\boldsymbol{\mu}_{x}$
    and variance
    $\sigma^{2}(\bs) = \bs^{T}\Sigma_{x}\bs$.
    Define $L(\bs) = \avg{\ell(z(\bs))}$, where $\ell(z)$ is a convex function with a unique minimum.
    
    Then, the optimum
    $\bs_{\min} = \operatorname*{arg\,min}_{\bs} L(\bs)$
    minimizes the ratio $r(\bs) \!\!=\!\! \sigma(\bs) / \mu(\bs)$. That is, $r(\bs_{\min})\le r(\bs), \forall\bs\in \mathbb{R}^d$.
\end{proposition}

\begin{proof}
    Let $Q:\mathbb{R}^d\to\mathbb{R}^2$ be mapping from weights to logit moments, $Q(\bs)=(\mu(\bs),\sigma(\bs)$. 
     Since $z$ is normally distributed, $L$ is completely determined by its moments $\mu(\bs)$ and $\sigma(\bs)$ and optimization in weight space is equivalent to a constrained optimization problem in moment space.
     Let $\Omega \subset \mathbb{R}^2$ be the set of attainable moments, ie the image of $\mathbb{R}^d$ under $Q$.
    
    Since both $\mu(\bs)$ and $\sigma(\bs)$ scale linearly in $\bs$ (e.g., $\mu(\lambda \bs)=\lambda \mu(\bs)$), each ray in weight space is mapped to a ray in $\Omega$. Therefore, $\Omega$ is a cone composed of rays emanating from the origin:
    \begin{equation}
        \Omega = \{(\mu, r \mu) \mid \mu > 0, \ r_1 \le r \le r_2 \},
    \end{equation}
    where we assumed without loss of generality that $\mu>0$ and $r_2$ may be infinite. 
    
    Using standard convexity arguments, it is easy to show that for any fixed mean $\mu$, $L$ is an increasing function of the spread $\sigma$ (see Lemma \ref{lemma:monotonicity_of_std} in the Appendix). Therefore, optimum in $\Omega$ is attained on its lower boundary, namely the ray with the minimal slope $r_1$. 
    Thus, $\bs_{\min}$ in mapped under $Q$ to the minimal $r = \sigma/\mu$. 
\end{proof}
\begin{corollary}[Quadratic Loss for any data]\label[corollary]{cor:quadratic_loss}
    For any data distribution, if the per-sample loss is quadratic, i.e., $\ell(z)=a(z-z^*)^2$, then the optimal direction $\hat{\bs}_{\min}$ minimizes the ratio $r = \sigma/\mu$.
\end{corollary}
\begin{proof}
    If $\ell$ is quadratic function then the loss is completely determined by the first and second moments of $z$.
    $$L(\bS)=\mathbb{E}\left[\ell(z(\bS))\right]=a\sigma^2(\bS)+(\mu(\bs)-z^*)^2$$
    Therefore, the argument of \cref{prop:gaussian_data} applies verbatim.
\end{proof}
We now state the main result: under these conditions, the implicit bias of logit regularization aligns the weights with Fisher's Linear Discriminant solution.
\begin{corollary}
    [LDA direction]\label[corollary]{cor:directional_invariance}
    If the data is Gaussian, the optimal weight direction is $\hat{\bs}_{\min}\propto \Sigma^{-1} \mu$. Similarly, if  $\lambda=d/N<1$ and the per-sample loss is quadratic, $\hat{\bs}_{\min} \propto \Sigma^{-1} \mu$ (where $\Sigma,\mu$ are the empirical mean and centered covariance).
\end{corollary}
\begin{proof}
    Identifying $r(\bs)^{-1}$ as the square root of the Fisher Criterion (i.e., $(\bs^\top \boldsymbol{\mu})^2/\bs^\top \Sigma \bs$), the result follows immediately from standard LDA theory. See \cref{app:complementary_proofs} for a direct proof.
\end{proof}
 Outside these idealized setting, we find empirically that direction of the solution remains close to this theoretical minimum, as we will show in the next section below.

\subsection{Consequence: Insensitivity to the Logit Penalty\label{sec:alpha_independence}}
We now turn to analyze the practical consequences this result.
Even though the direction of the solution would align with LDA only when the data is purely Gaussian or the per-sample loss is completely quadratic, in practice, we expect this to hold approximately (i.e., the deviation from this direction would not be large) for a much broader set of cases: As discussed in~\cref{sec:clustering}, logit reguarlization tends to cluster the logits around the loss minimum. Therefore, locally, the optimal per-sample loss may be approximated as quadratic.
This suggest, as we will also demonstrate numerically below, that the optimal generalization accuracy $\mathcal{A}_{\mathrm{gen}}^{\min}$ will be quite insensitive to the specific per-sample loss $\ell(z)$, and in particular the logit regularization function $f(z)$.

\textbf{Geometric Intuition.}
First, we will provide intuition from another perspective. As stated above, minimization of the loss now requires placing logits as tightly as possible around $z^*$, the minimum of $\ell(z)$. This implicitly decouple the optimization into two distinct tasks:
\begin{enumerate}
    \item \textbf{Shape Optimization (Direction):} Making the cluster as tight as possible, by choosing a direction $\hat{\bs}=\bs/\norm{\bs}$ minimizing the CV $r=\sigma/\mu$.
    \item \textbf{Location Optimization (Norm):} Scaling $\|\bs\|$ to shift the logit mean $\mu$ to the proximity of $z^*$.
\end{enumerate}
Approximately, the first task is determined solely by the dataset geometry and not by the per-loss function $\ell(z)$. Since generalization accuracy depends only on $\hat{\bs}$ (and not $\norm{\bs}$) it should be invariant to the choice of regularization function (while the norm would be very sensitive to the per-sample minimum).

\textbf{Empirical Evidence.}
We will now provide numerical evidence for this insensitivity. In \Cref{fig:alpha_independence}, we consider a binary classification experiment where the data is the mean vector $\pm (1,0,0,\dots)$ corresponding to the class label (arbitrarily aligned with the first axis), plus a random noise vector where each component is drawn independently from a Student-t distribution with degrees of freedom $\nu$. Here, the logit-regularization function is taken to be $f(z)=z^2$.
This is an illustrative example of a more general data model which will be discussed in detail in~\cref{sec:signal_noise_decomp}.
The left panel shows the cosine similarity between the learned $\hat{\bs}_{\min}$ and the true feature axis. We observe a sharp transition from the unregularized state ($\alpha=0$) to $\alpha>0$, indicating the \emph{change of the implicit bias}, followed by stable plateau for $\alpha > 0$. The plateau is flatter for larger $\nu$ (approaching Gaussianity) or larger $\alpha$ (approaching a quadratic loss). In stark contrast, the right panel shows that the norm $\|\bs_{\min}\|$ varies significantly with $\alpha$. This confirms that logit regularization modifies the scale of the weights to satisfy the per-sample minimum, while the direction (and thus the accuracy) remains robust. For additional numerical evidence, see~\cref{app:additional_numerical_evidence}.

\begin{figure}[t]
    \centering
    \includegraphics[scale=0.47]{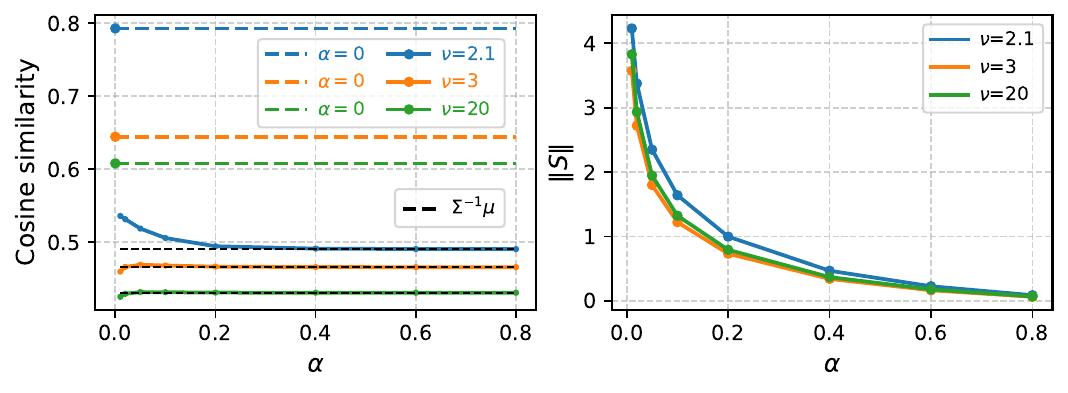}
    \caption{\label{fig:alpha_independence}
    \textbf{Change in the implicit bias and robustness to $
    \alpha$.} Analysis of the optimal weights $\bs_{\min}$ on Student's t-distributed data ($\nu \in \{2.1, 3, 20\}$).
    \textbf{Left panel:} Cosine similarity between $\bs_{\min}$ and the feature-axis, where the data is drawn from a Student's t-distribution for several values of $\nu$. The dashed lines indicates the unregularized values ($\alpha=0$). We observe an abrupt shift from $\alpha=0$ to $\alpha>0$, followed by a plateau that becomes flatter as $\nu$ increases (approaching a Gaussian distribution). On black dashed line we plot for comparison the limiting LDA expected value, $\Sigma^{-1}\mu$, where $\mu$ and $\Sigma$ are the empirical mean and centered covariance. \textbf{Right panel:} The norm $\|\bs_{\min}\|$, which, in contrast, exhibits a clear dependence on $\alpha$. See \cref{app:numerical_details} for more details about the numerical setup.}
\end{figure}

\section{Signal-plus-Noise Decomposition}
\label{sec:signal_noise_decomp}
Building on the geometric insight of logit clustering, we now derive several surprising consequences of logit regularization when the data distribution 
has a signal-plus-noise structure.

\subsection{Data model}
We model the data geometry via a decomposition into a signal subspace and an orthogonal noise subspace. The input $\boldsymbol{x} \in \mathbb{R}^d$ is generated according to
\begin{equation}
    \boldsymbol{x} = y \mu_f \boldsymbol{e}_1 + \boldsymbol{\xi},
\end{equation}
where $y=\pm1$ is the label, $\boldsymbol{e}_1$ denotes the signal direction, and $\mu_f$ represents the signal magnitude. The noise vector $\boldsymbol{\xi}$ is drawn from a zero-mean distribution with anisotropic scaling. Specifically, we define $\boldsymbol{\xi} = \sigma_f \xi_f \boldsymbol{e}_1 + \sigma_n \boldsymbol{\xi}_{\perp}$, where $\xi_f$ and the components of $\boldsymbol{\xi}_{\perp}$ are drawn from distributions $\mathcal{D}_f$ and $\mathcal{D}_n$ (signal and noise) respectively, with zero mean. Here, $\sigma_f$ and $\sigma_n$ act as mixture weights. We note that in the case of $\sigma_f=\sigma_n$ and Gaussian noise, our data generation process coincides with the noisy Gaussian mixture classification problem studied in~\citet{mignacco2020regularization} and subsequent works.

Unless explicitly stated otherwise (e.g., in the ResNet-18 validation), we 
take the noise distributions to be Gaussian and align the signal with the $x_1$-axis with unit amplitude (i.e., $\mu_f=1$). Under this assumption, $\sigma_f$ and $\sigma_n$ represent the standard deviations of the signal-aligned and orthogonal noise components, respectively. Formally, the noise vector $\boldsymbol{\xi}$ has independent entries satisfying $\xi_1 \sim \mathcal{N}(0, \sigma_f^2)$ and $\xi_i \sim \mathcal{N}(0, \sigma_n^2)$ for all $i > 1$. The full details regarding the setup of each figure appear in Appendix \ref{app:numerical_details}.

\subsection{Empirical Motivation}
\label{sec:empirical_motivation}

To motivate our construction, we show that the penultimate representations of a NN trained on real-world images exhibit a geometric structure that aligns with our signal-plus-noise decomposition.
Analyzing the second-order statistics, we show that the noise structure is anisotropic with respect to the task as in our suggested model.

Specifically, we consider the penultimate layer of a ResNet-18 trained on ImageNet and evaluated on CIFAR-10 (clean) versus CIFAR-10C \citep{hendrycks2019benchmarking} (noisy). We focus on a binary task (Planes vs. Cats), and calculate for each class the empirical class means $\bmu_c =\frac{1}{N_c} \sum_i \bx_i$ and the disjoint empirical centered covariance matrices $\Sigma_c = \frac{1}{N_c} \sum_{i} (\bx_i - \bmu_c)(\bx_i - \bmu_c)^\top$.

To understand the noise geometry, we project these covariances matrices onto two subspaces: the signal subspace, spanned by the difference of the class means $\boldsymbol{v} = \bmu_1 - \bmu_2$, and its orthogonal complement. Computing the standard deviation along the signal axis and the root-mean-square (RMS) of the eigenvalues in the orthogonal space, we get empirical equivalents of $\sigma_f$ and $\sigma_n$.
As shown in \cref{fig:feature_geometry_motivation}, the intra-class variance along the signal direction ($\sigma_f$) differs significantly from the variance in the orthogonal directions ($\sigma_n$).
Furthermore, when the input images are corrupted (using CIFAR-10C), these noise components grow.

These observations suggests that a single isotropic noise parameter cannot fully capture the geometry of real-world representations, which justify our choice of a more expressive two-parameter model (defined by $\mathcal{D}_f$ and $\mathcal{D}_n$).
We will return to this exact empirical setting in \cref{sec:resnet_validation} to validate our main theoretical predictions on real-world data.

\begin{figure}[t!]
    \centering
    \begin{subfigure}[t]{0.24\textwidth}
        \centering
        \includegraphics[width=\textwidth]{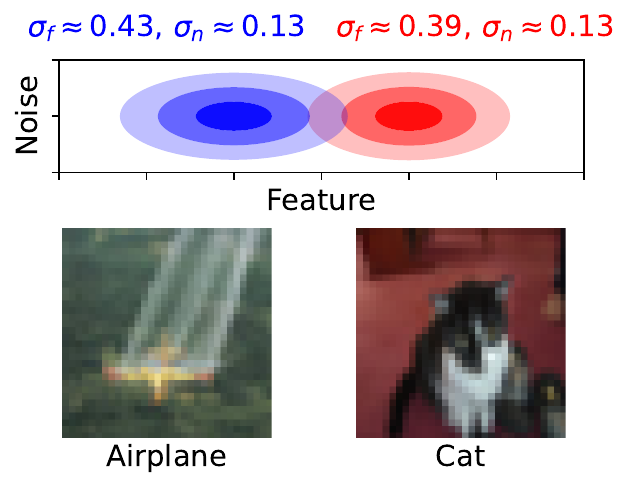}
    \end{subfigure}
    \begin{subfigure}[t]{0.228\textwidth}
        \centering
        \includegraphics[width=\textwidth]{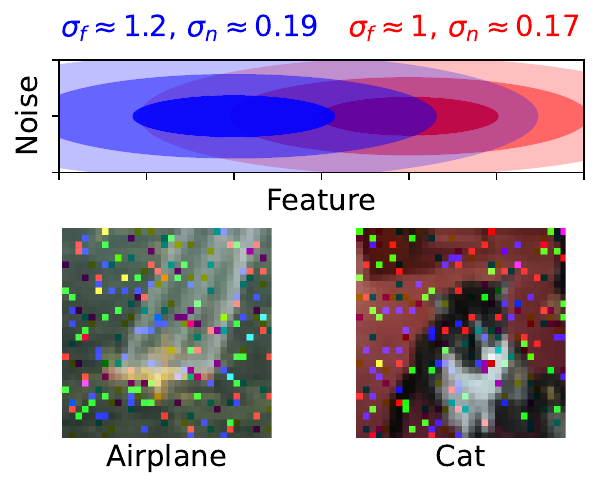}
    \end{subfigure}
    \begin{subfigure}[t]{0.428\textwidth}
        \centering
        \includegraphics[width=\textwidth]{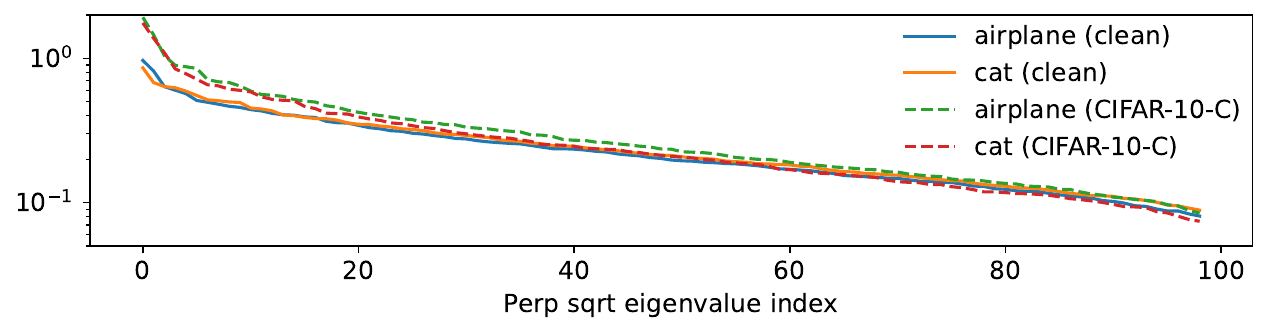}
    \end{subfigure}
    \caption{\textbf{Effect of Input Noise on Feature Geometry.} 
    \textbf{Top Row:} Visualization of penultimate layer features for two classes (Planes vs. Cats) from ResNet-18. The horizontal axis represents the signal direction (connecting class means), while the vertical axis represents the effective orthogonal noise radius. Left: Clean CIFAR-10 data. Right: Noisy CIFAR-10C data.
    \textbf{Bottom Row:} The eigenspectrum of the orthogonal noise covariance matrix. For more details, see Appendix \ref{app:numerical_details}.
    }
    \label{fig:feature_geometry_motivation}
    \vspace{-10pt}
\end{figure}

\section{The Noiseless Feature Regime}

\label{sec:noiseless_regime}

We begin by investigating logistic regression under the signal-plus-noise data decomposition, in the regime where the feature noise in the signal direction vanishes ($\sigma_f=0$), yielding several surprising results.
While idealized, the phenomena derived below persist approximately whenever the signal noise is sufficiently small ($\sigma_f \ll \mu_f, \sigma_n$).

\subsection{Shift of interpolation threshold}
We analyze the high-dimensional proportional limit where $d, N \to \infty$ while the ratio $\lambda = d/N$ remains fixed. In the unregularized setting, the interpolation threshold is known to occur at $\lambda_c = 1/2$ \citep{gardner1989unfinished}. The reason is tied to  separability: for $\lambda > 1/2$, points in general position are linearly separable regardless of their labels. This geometric freedom allows the optimizer to satisfy the margin constraints by discovering separating directions distinct from the ``true'' signal $\boldsymbol{S}=\boldsymbol{e}_1$, leading to overfitting~\citep{beck2025grokking}.
In contrast, we will show that the introduction of \textbf{any} convex logit regularization will shift abruptly the interpolation threshold to $\lambda_c = 1$.

\begin{proposition}
    Let $\sigma_f=0$ and $\alpha>0$. Then, the minimizer of the training loss achieves perfect generalization accuracy if and only if $\lambda = d/N < 1$.
\end{proposition}

\begin{proof}
    With $\sigma_f=0$, the signal component is deterministic: $y_i \boldsymbol{e}_1^\top \boldsymbol{x}_i = \mu_f$ for all $i$. Since $\alpha > 0$, the per-sample loss is strictly convex with a unique global minimum $z^*$.
    
    We can construct a weight vector aligned purely with the signal, $\boldsymbol{S} = (z^*/\mu_f)\boldsymbol{e}_1$, such that $y_i \boldsymbol{S}^\top \boldsymbol{x}_i = z^*$ for all samples. This configuration achieves the global unconstrained lower bound of the total loss (the sum of per-sample minima). This solution is unique as the data matrix is full rank ($N > d$), and since it aligns perfectly with $\boldsymbol{e}_1$, it yields $100\%$ test accuracy.
    
    Conversely, for $\lambda > 1$, the data matrix has a non-trivial null space. The loss minimizer takes the form $\boldsymbol{S}_{\min} = (z^*/\mu_f)\boldsymbol{e}_1 + \boldsymbol{S}_{\perp}$, where $\boldsymbol{S}_{\perp} \neq \boldsymbol{0}$ lies in the null space of the training data, degrading test accuracy.
\end{proof}
\begin{figure}[t!]
    \centering
    \includegraphics[scale=0.47]{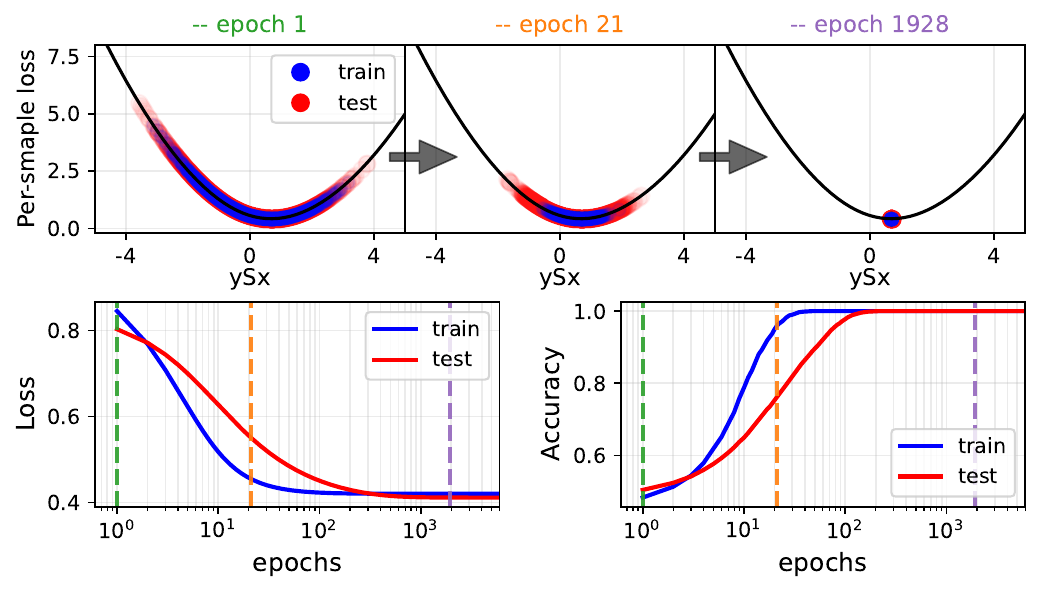}
    \caption{\textbf{Perfect generalization with zero feature noise ($\sigma_f=0$).} 
    \textbf{Top:} Evolution of the logit distribution across three training epochs. At convergence, all logits collapse to a single target value $z^*$. \textbf{Bottom:} Loss and accuracy over time. Vertical dashed lines indicate the epochs corresponding to the top snapshots.
    }
    \label{fig:logit_clustering_zero_sigma_f}
\end{figure}

\cref{fig:logit_clustering_zero_sigma_f} demonstrates the fact that the logits collapse to a single point in this regime. Notably, this result is independent of the regularization function $f(z)$ (consistent with our previous discussion) and the orthogonal noise scale $\sigma_n$ (which is not unique to $\sigma_f=0$ as discussed below). This suggests that for sufficiently small $\sigma_f$, logit regularization is highly beneficial, effectively extending the generalization regime to $\lambda < 1$. We examine the case of non-negligible $\sigma_f$ in the following section.

\subsection{Grokking Dynamics.}

Grokking (i.e., delayed generalization accompanied by non-monotonic test loss) has drawn significant recent attention. Developing simple analytical models that exhibit grokking is crucial for isolating the mechanisms that induce it. Previous work has identified grokking in the unregularized version of this setup, occurring near the critical threshold $\lambda_c=1/2$~\citep{beck2025grokking}.
Notably, we find that logit regularization can also induce grokking dynamics across the broad region $1/2 < \lambda < 1$ --- see \cref{fig:grokking}). This occurs because the interpolation threshold is shifted to $\lambda_c=1$. Consequently, the system exhibits delayed generalization as $\alpha \to 0$: the model initially overfits (in the direction of the unregularized implicit bias) before eventually converging to the generalizing solution. The timescale of this delay (``grokking time'') diverges as $\alpha \to 0$, consistent with previous works showing that the grokking time diverges as some control parameter is close to criticality~\citep{levi2024grokkinglinear}.

\begin{figure}[t!]
\begin{centering}
\includegraphics[scale=0.38]{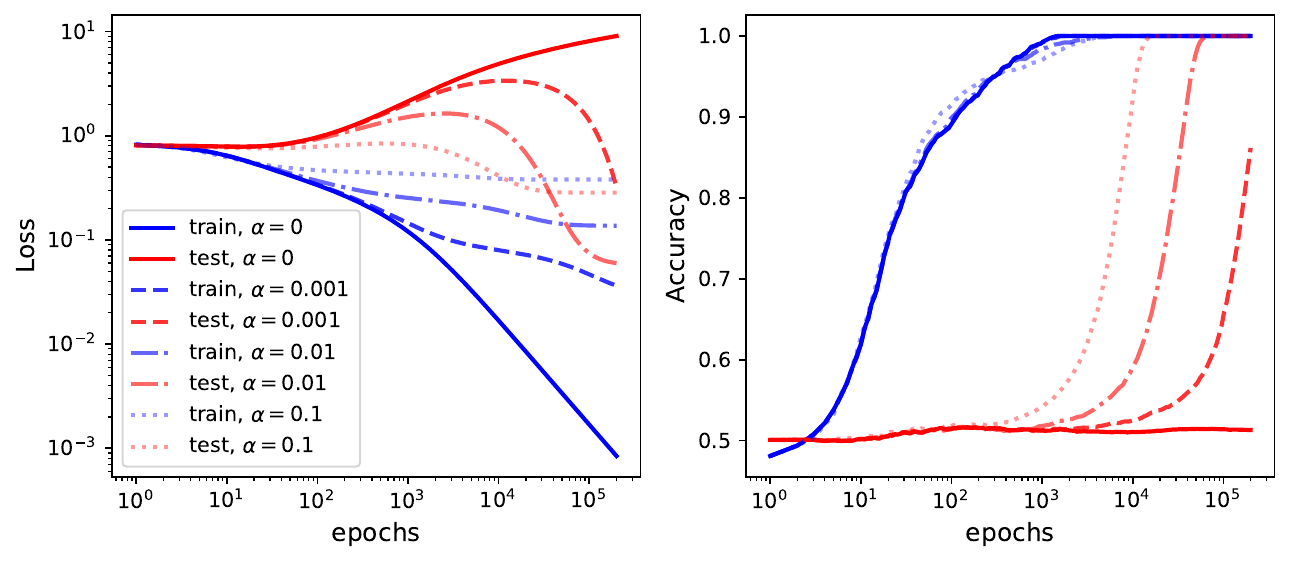}
\par\end{centering}
\caption{\label{fig:grokking}\textbf{Grokking.} Loss and accuracy curves for $\sigma_{f}=0$ and $\lambda=0.7$. While the unregularized model ($\alpha=0$) overfits, logit regularization induces grokking (delayed generalization), with the delay diverging as $\alpha$ approaches zero.}
\end{figure}

\section{Invariance to the Noise Scale}

In this section, we show that in the presence of logit-regularization, the optimal generalization accuracy becomes independent of the orthogonal noise scale $\sigma_n$. We then use this result to map the specific regimes where logit regularization outperforms the unregularized baseline.

\subsection{Proof of the Invariance}
\begin{proposition}
\label{prop:sigma_n_independence}
Let $\alpha > 0$, $\lambda = d/N < 1$, and fix $\sigma_f, \mu_f$. Then, the optimal generalization accuracy $\mathcal{A}^{\min}_{\mathrm{gen}}$ is strictly independent of $\sigma_n$.
\end{proposition}

\begin{proof}
    This follows from a simple scaling argument. 
     Let $\boldsymbol{S}_{\min}=(S_{1,\min}, \boldsymbol{S}_{\perp,\min})$ denote the minimizer. The loss depends exclusively on the scalar logits $z_i = S_{1,\min}x_{i,1} + \boldsymbol{S}_{\perp,\min}^{T}\boldsymbol{x}_{i,\perp}$. Consider a rescaling of the noise variance $\sigma_{n} \to \beta\sigma_{n}$, which implies $\boldsymbol{x}_{i,\perp} \to \beta\boldsymbol{x}_{i,\perp}$. To maintain the optimal logit values $z_{\min}$ (and thus minimize the loss), the weights must scale inversely: $\boldsymbol{S}_{\perp,\min} \to \beta^{-1}\boldsymbol{S}_{\perp,\min}$. Consequently, the product $\boldsymbol{S}_{\perp,\min}^{T}\boldsymbol{x}_{i,\perp}$ remains invariant. Since the distribution of the test logits depends on this same product, the test accuracy is invariant to $\sigma_n$.
\end{proof}

Notably, while $\mathcal{A}^{\min}_{\mathrm{gen}}$ is independent of $\sigma_n$, the direction $\bs_{min}$ is not. The cosine similarity with the feature axis $\boldsymbol{e}_1$, denoted by $\rho$, explicitly depends on the noise. Since $\|\boldsymbol{S}_{\perp,\min}\| \propto \sigma_n^{-1}$, the alignment behaves as:
\begin{equation}
    \rho_{\min} = \frac{S_{1,\min}}{\|\boldsymbol{S}_{\min}\|} = \frac{1}{\sqrt{1 + (C/\sigma_n)^2}},
    \label{eq:rho_min}
\end{equation}
where $C$ depends on $\sigma_f$ but is independent of $\sigma_n$ (see Appendix \ref{subsec:cosine_similarity_dependence}).

To demonstrate how can it be that the accuracy remains constant despite the changing direction of $\boldsymbol{S}$, we examine the case of Gaussian data. The analytical expression for generalization accuracy in this case is:
\begin{equation}
    \mathcal{A}_{\mathrm{gen}}
    =
    \frac{1}{2}\left[
    1+\mathrm{erf}\!\left(
    \frac{\rho \mu_{f}}{\sqrt{2\bigl(\rho^{2}\sigma_{f}^{2}+(1-\rho^{2})\sigma_{n}^{2}\bigr)}}
    \right)
    \right].
    \label{eq:A_gen_vs_rho}
\end{equation}
Substituting $\rho_{\min}$ from \cref{eq:rho_min} into \cref{eq:A_gen_vs_rho} cancels the $\sigma_n$ terms in the denominator, rendering the final accuracy explicitly independent of the noise scale.
See \cref{fig:A_test_vs_lambda} in the appendix for more empirical evidence to the $\sigma_n$-invariance.

\subsection{When does logit regularization help?}
Since the regularized accuracy $\mathcal{A}_{\alpha>0}^{\min}$ is constant with respect to $\sigma_n$ while the unregularized accuracy $\mathcal{A}_{\alpha=0}^{\min}$ change, one may find a critical noise threshold where these curves intersect. This intersection defines the point at which logit regularization becomes advantageous. For Gaussian data, we illustrate this crossover in the left panel of \cref{fig:sigma_n_independence_and_phase_diagram}. By tracing this threshold across varying $\sigma_f$, we construct the phase diagram shown in the right panel, presenting the regime where regularization improves performance.

\begin{figure}[t!]
    \centering
    \includegraphics[scale=0.32]{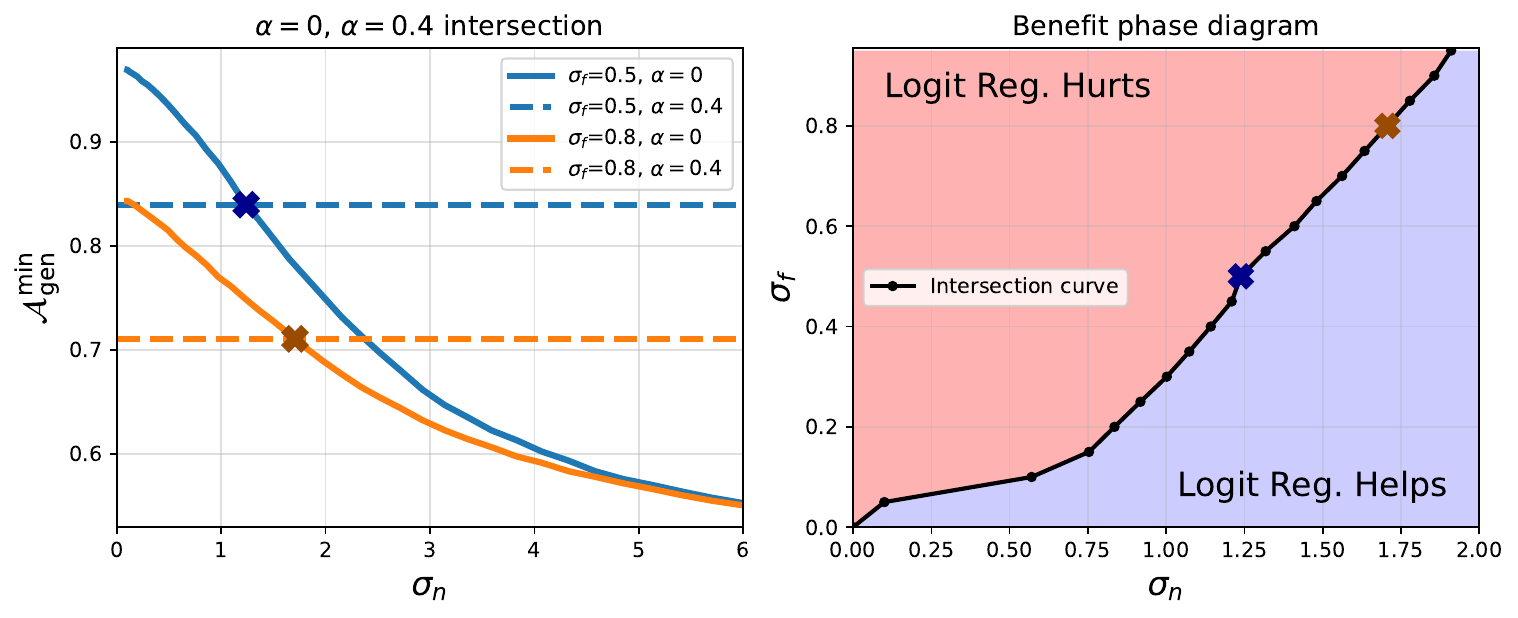}
    \caption{ \textbf{Benefit phase diagram.}
    \textbf{Left:} Final generalization accuracy as a function of $\sigma_n$, comparing the unregularized ($\alpha=0$) and regularized ($\alpha>0$) cases, for $\lambda=0.7$. \textbf{Right:} Phase diagram showing regions where logit regularization is beneficial. The boundary curve is constructed from the intersection points (as shown in the left panel) for different values of $\sigma_f$. For illustration, the intersection points of the unregularized and regularized curves for $\sigma_f=0.5, 0.8$ (left) are marked as ``X''s of corresponding colors on the boundary curve (right). For an equivalent figure in the $\lambda<1/2$ regime, see \cref{fig:sigma_n_independence_and_phase_diagram_v2} in the appendix.}
    \label{fig:sigma_n_independence_and_phase_diagram}
\end{figure}

\section{Generalization to Multi-class Setting}
\label{sec:multiclass}

In this section, we extend our analysis to the $K$-class setting, demonstrating that the core geometric mechanism of logit clustering remains the same.

\subsection{Setup}

We consider a linear classifier producing logits $\boldsymbol{z} = W\boldsymbol{x} + \boldsymbol{b}$, where $W \in \mathbb{R}^{K \times d}$. The predicted class is $c = \mathrm{argmax}_k z_k$.
Let $\boldsymbol{y}$ be the one-hot label vector for class $c$. We minimize the empirical loss $\mathcal{L} = \frac{1}{N}\sum_{i} \ell(\boldsymbol{z}^{(i)}, \boldsymbol{y}^{(i)})$, where the per-sample loss is defined as:
\begin{equation}\label{eq:loss_logit_regularization_multiclass}
    \ell(\boldsymbol{z}, \boldsymbol{y}) = (1-\alpha)\ell_{\mathrm{CE}}(\boldsymbol{z}, \boldsymbol{y}) + \alpha f(\boldsymbol{z}).
\end{equation}
Here, $\ell_{\mathrm{CE}}(\boldsymbol{z}, \boldsymbol{y}) = -\sum_k y_k \log p_k$ is the standard cross-entropy loss, and $f(\boldsymbol{z})$ is a symmetric, convex regularization term. We note that standard LS is recovered by choosing (See Appendix \ref{appendix:label_smoothing}):
\begin{equation}
    f_{\mathrm{LS}}(\boldsymbol{z}) = \log\left(\sum_{k=1}^K e^{z_k}\right) - \frac{1}{K}\sum_{k=1}^K z_k.
\end{equation}

To extend the signal-plus-noise decomposition, we assume the class means $\{\boldsymbol{\mu}_c\}_{c=1}^K$ are affinely independent, spanning a $(K-1)$-dimensional centered signal subspace $\mathcal{S}=\mathrm{span}\{\boldsymbol{\mu}_c-\bar{\boldsymbol{\mu}}\}_{c=1}^K$. The input $\boldsymbol{x}$ for a given class $c$ is modeled as $\boldsymbol{x} = \boldsymbol{\mu}_c + \boldsymbol{\xi}$. The noise vector $\boldsymbol{\xi}$ decomposes into signal-aligned and orthogonal components:
\begin{equation}
    \boldsymbol{\xi} = \sigma_f \boldsymbol{\eta}_f + \sigma_n \boldsymbol{\eta}_n,
\end{equation}
where $\boldsymbol{\eta}_f \in \mathcal{S}$ and $\boldsymbol{\eta}_n \in \mathcal{S}^\perp$ are drawn from standard isotropic distributions $\mathcal{D}_f$ and $\mathcal{D}_n$, scaled by $\sigma_f$ and $\sigma_n$, respectively.

\begin{figure}[t]
    \centering
    \includegraphics[scale=0.3]{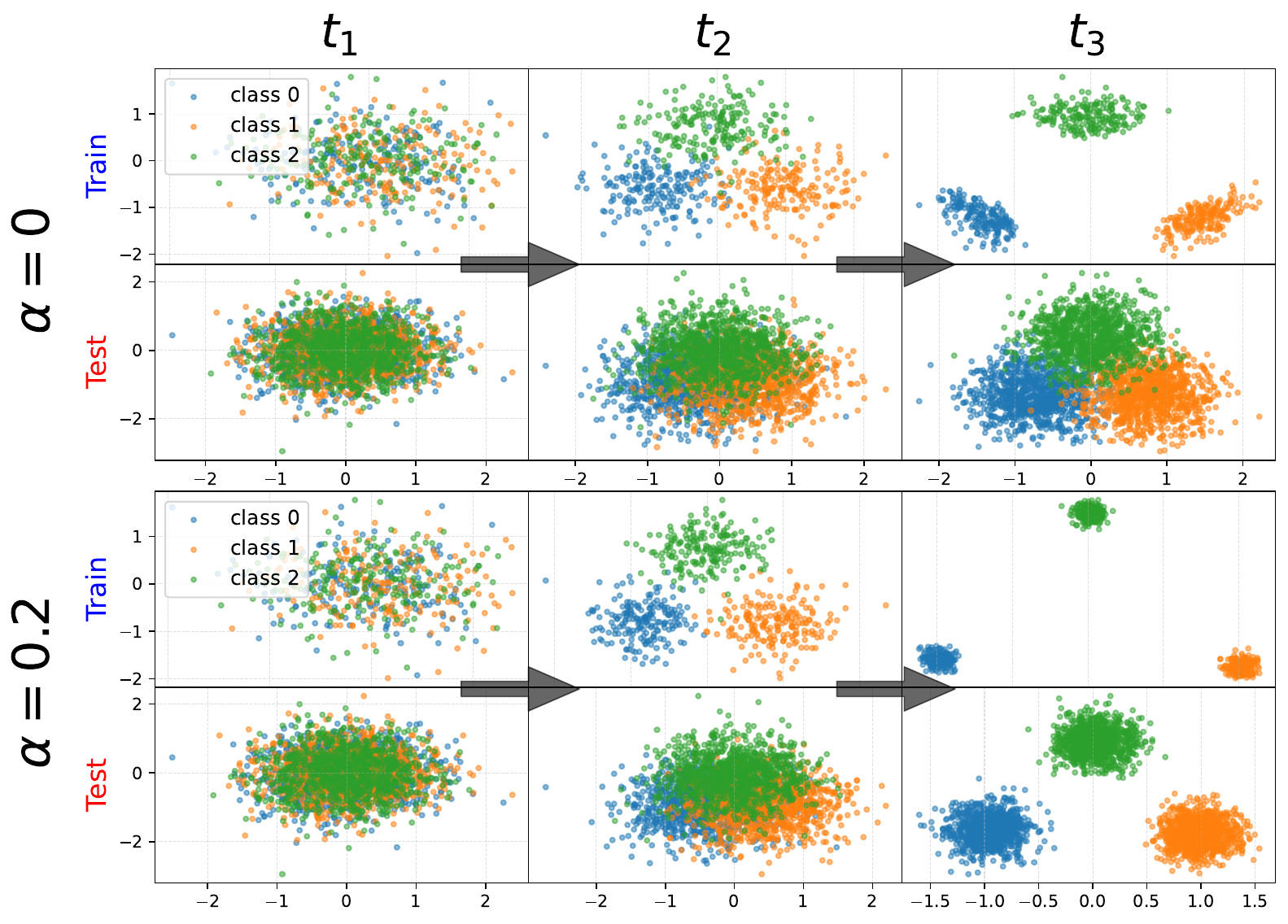}
    \caption{\textbf{Multi-class Logit Clustering.} 
    Projection of the data from three classes onto the subspace spanned by the relative weight differences ($S_1-S_2$ and $S_1-S_3$). The parameter is set to $\lambda=d/N=0.7$: since $\lambda > 1/2$, the data is separable.
    \textbf{Left ($\alpha=0$):} Standard cross-entropy separates classes linearly, maximizing the margin.
    \textbf{Right ($\alpha=0.2$):} Logit regularization enforces tight clustering around the simplex vertices.}
    \label{fig:multiclass_clustering}
\end{figure}

\begin{figure}[t!]
\centering\includegraphics[width=1.01\linewidth]{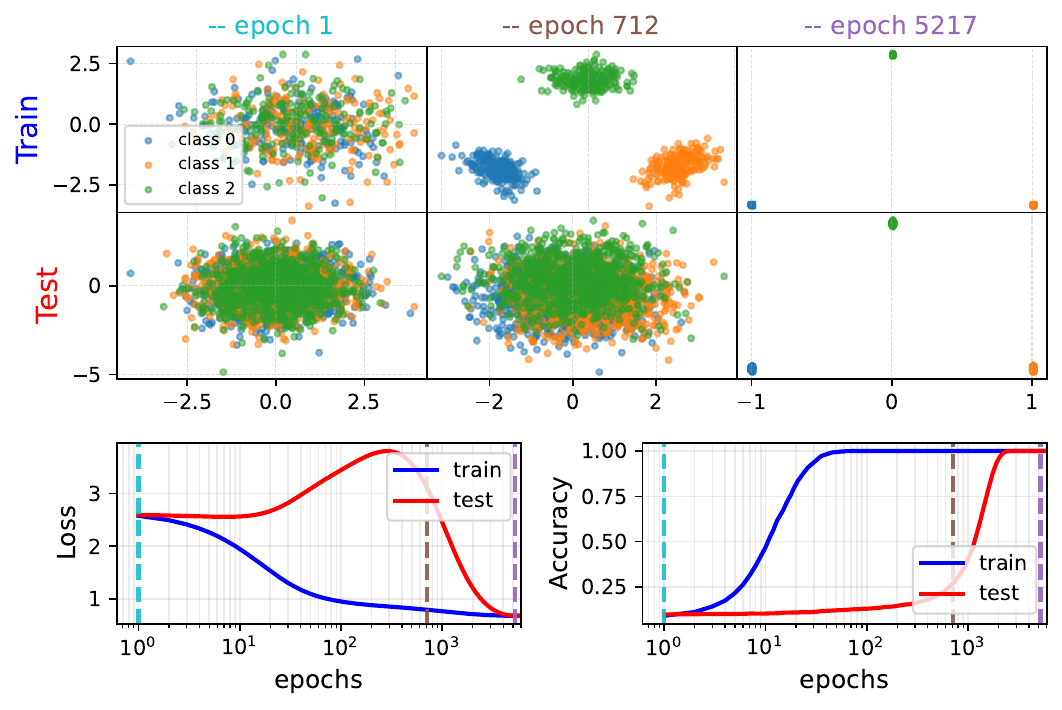}
\caption{
\textbf{Grokking in Multi-class and Small Logit Regularization.}
Dynamics in the multi-class case, for $\sigma_f=0$, $\lambda=0.7$ and $\alpha=0.04$.
\textbf{Top:} Logit projections at three time-steps. The model initially learns the overfitting solution before eventually collapsing the minima of the simplex vertices, corresponding to the generalization phase.
\textbf{Bottom:} Evolution of test accuracy and loss. We observe the characteristic grokking signature: the test loss initially rises (overfitting) before descending, and generalization accuracy is delayed.
}
\label{fig:multiclass_grokking}
\end{figure}

\subsection{Simplex Clustering minimization}

As in the binary classification case, the introduction of regularization ($\alpha > 0$) fundamentally alters the optimization landscape. 
The per-sample loss now possesses a finite minimum. Explicitly writing the per-sample loss for class $c$ (where $\boldsymbol{y}=\boldsymbol{y}_c$ is a one-hot vector with $1$ at index $c$), we obtain:
\begin{equation}
    \ell(\boldsymbol{z}, \boldsymbol{y}=\boldsymbol{y}_c) = (1-\alpha)
    \left( \log\left(\sum_{j}e^{z_{j}}\right) -z_{c}\right) 
    + \alpha f(\boldsymbol{z}).
\end{equation}
For each class $c$, there exists a unique target logit vector $\boldsymbol{z}^*_{(c)}$ that minimizes this per-sample loss. Due to permutation symmetry, this vector takes a specific form: a high value $z_{\mathrm{high}}$ at index $c$, and a uniform lower value $z_{\mathrm{low}}$ elsewhere. This represents the maximal certainty permitted by the regularization, corresponding to a probability distribution that peaks at class $c$ (with a value determined by $\alpha$) and uniform elsewhere.

The set of target vectors $\{\boldsymbol{z}^*_{(1)}, \dots, \boldsymbol{z}^*_{(K)}\}$ forms the vertices of a symmetric simplex in logit space. Consequently, the optimization objective transforms into a \emph{clustering} problem: the model seeks a weight matrix $W$ that collapses the logits of each class $c$ as tightly as possible around its corresponding vertex $\boldsymbol{z}^*_{(c)}$. 

This behavior is visualized in the numerical experiment presented in  \cref{fig:multiclass_clustering}. 
We assume Gaussian distributions for both signal and noise components. 
The model is trained with a quadratic regularization term $f(\boldsymbol{z})=\|\boldsymbol{z}\|^2$.
We then adopt the visualization scheme introduced in \citet{muller2019when}, projecting the data of three classes onto the plane spanned by the class templates at three distinct time-steps during training. In the unregularized regime ($\alpha=0$), the (projected) logits scatter to maximize separability. In contrast, regularization ($\alpha>0$) forces them to cluster at the simplex vertices.

Most of our key results from the binary setting extend naturally to the multi-class case: see \cref{app:multiclass_additional_details} for details. For example, \cref{fig:multiclass_grokking} demonstrates our main finding regarding the noiseless regime: for $1/2<\lambda<1$ and small $\alpha$, we observe grokking. First, the training samples become fully separated while the model overfits the test set. At later times, the logits switch their implicit bias due to logit-regularization, collapsing to the vertices of the simplex and obtaining perfect generalization.

\begin{figure}[!t]
    \centering
    \includegraphics[scale=0.47]{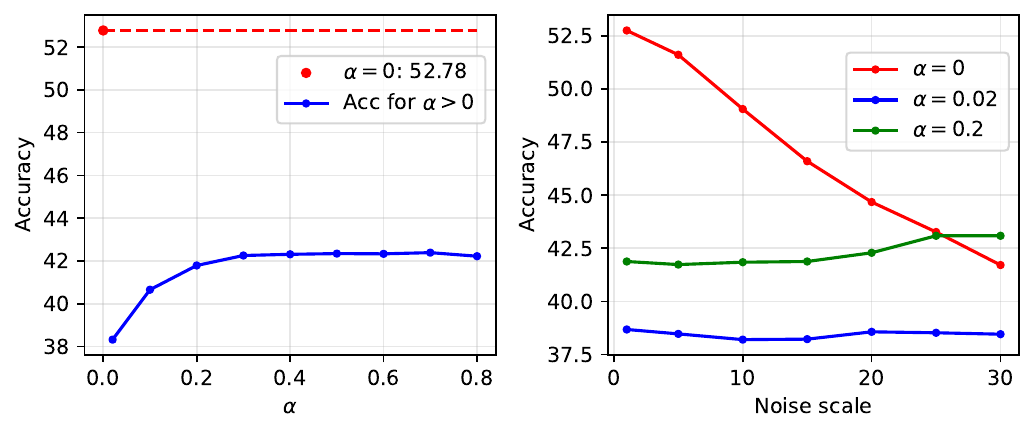}
    \caption{\textbf{Validation on ResNet-18 Penultimate Features.} 
    \textbf{Left:} Generalization accuracy vs. regularization strength $\alpha$. We observe a jump followed by saturation, confirming directional stability.
    \textbf{Right:} Accuracy vs. synthetic scaling of the orthogonal noise component. Logit regularization ($\alpha > 0$) makes the model insensitive to the noise scale, in contrast to the unregularized baseline.}
    \label{fig:Penultimate_resnet18_results}
\end{figure}

\section{Validation on Penultimate Layer Embeddings}
\label{sec:resnet_validation}

The analysis above was carried out in a linear setting with an idealized signal-plus-noise decomposition. We now validate two main qualitative predictions --- the sharp transition from $\alpha=0$ to $\alpha>0$ and the invariance to orthogonal noise scaling --- in a more realistic feature space.

\textbf{Setup.}
We use the same ResNet-18 penultimate-layer embeddings introduced in \cref{fig:feature_geometry_motivation}. We freeze the representation and train only a linear classifier on top, using the loss in \cref{eq:loss_logit_regularization_multiclass}, with quadratic logit regularization $f(\bz)=\|\bz\|^2$. See \cref{app:numerical_details} for more details.

\textbf{$\alpha$-Insensitivity.}
We examine the sensitivity of the optimal generalization accuracy to the regularization strength $\alpha$, see left panel of \cref{fig:Penultimate_resnet18_results}. As predicted in \cref{sec:alpha_independence}, the accuracy jumps significantly and enters a stable plateau. The deviation from it at smaller $\alpha$ is due
to the fact that the data is not a Gaussian.
 
\textbf{$\sigma_n$-independence via synthetic orthogonal rescaling.}
To directly test the invariance in Proposition~\ref{prop:sigma_n_independence}, we synthetically rescale only the component of each embedding that is orthogonal to the subspace spanned by the class means. This changes the effective $\sigma_n$ without altering the class-mean geometry. The right panel of \cref{fig:Penultimate_resnet18_results} shows that the unregularized classifier ($\alpha=0$) degrades substantially under this rescaling, whereas the regularized classifier ($\alpha>0$) is more robust, according to our expectations.

\section{Conclusion}

In this work, we have characterized the implicit bias of convex logit regularization, such as label smoothing, in linear classification. We showed that such regularization results in \emph{logit clustering} around the finite minima of the per-sample loss function. This provides a straightforward geometric explanation for the known label-smoothing effect, and demonstrates that the phenomenon is not unique to Label Smoothing, but is a general consequence of any loss that can be decomposed as a sum of per-sample losses with a finite minimum.

Remarkably, this clustering objective is intimately connected to classical statistical discrimination. We proved that for Gaussian data or quadratic loss functions, logit regularization drives the weight vector to align exactly with \emph{Fisher's Linear Discriminant} ($\Sigma^{-1}\boldsymbol{\mu}$). Empirically, we demonstrated that the direction remains close to this solution even beyond the idealized conditions. This result bridges a conceptual gap, showing that modern ``soft target'' techniques like Label Smoothing effectively recover the optimal linear discriminator for Gaussian data. These findings also suggest that the benefits of Label Smoothing are not unique to the specific functional form used in practice but are shared by a broad class of convex logit penalties.

Leveraging a signal-plus-noise decomposition, we also demonstrated that logit regularization shifts the interpolation threshold to $\lambda=1$ (yielding perfect generalization in regimes where unregularized models fail) and induces grokking dynamics under weak regularization. Finally, we proved that regularized accuracy is invariant to orthogonal noise scaling, allowing us to map the precise phase space where regularization is beneficial.

\textbf{Limitations and Future Work.} Our analysis relies on a linear framework. While we validated our findings on fixed ResNet embeddings, extending the logit clustering perspective to end-to-end training (where the representation is learned jointly) remains a primary challenge. Additionally, while we characterized the implicit bias for Gaussian data as the minimization of the coefficient of variation, identifying the analogous geometric invariant for general, heavy-tailed distributions remains an open question. Future research should investigate how this ``LDA-like'' bias interacts with the inductive bias of deep networks during feature learning, as well as the efficacy of this broader class of convex penalties in practical deep learning settings.

\bibliography{bibib}

\begin{thebibliography}{45}
\providecommand{\natexlab}[1]{#1}
\providecommand{\url}[1]{\texttt{#1}}
\expandafter\ifx\csname urlstyle\endcsname\relax
  \providecommand{\doi}[1]{doi: #1}\else
  \providecommand{\doi}{doi: \begingroup \urlstyle{rm}\Url}\fi

\bibitem[Beck et~al.(2025)Beck, Levi, and Bar-Sinai]{beck2025grokking}
A.~Beck, N.~I. Levi, and Y.~Bar-Sinai.
\newblock Grokking at the edge of linear separability.
\newblock In \emph{International Conference on Machine Learning (ICML)}, 2025.
\newblock URL \url{https://arxiv.org/abs/2410.04489}.

\bibitem[Celentano et~al.(2021)Celentano, Cheng, and Montanari]{celentano2021fom}
M.~Celentano, C.~Cheng, and A.~Montanari.
\newblock The high-dimensional asymptotics of first order methods with random data.
\newblock \emph{arXiv preprint arXiv:2112.07572}, 2021.
\newblock URL \url{https://arxiv.org/abs/2112.07572}.

\bibitem[Chandrasegaran et~al.(2022)Chandrasegaran, Tran, Zhao, and Cheung]{chandrasegaran2022revisiting}
K.~Chandrasegaran, N.-T. Tran, Y.~Zhao, and N.-M. Cheung.
\newblock Revisiting label smoothing and knowledge distillation compatibility: What was missing?
\newblock In \emph{Proceedings of the International Conference on Machine Learning (ICML)}, pages 2890--2916. PMLR, 2022.

\bibitem[Cover(1965)]{cover1965geometrical}
T.~M. Cover.
\newblock Geometrical and statistical properties of systems of linear inequalities with applications in pattern recognition.
\newblock \emph{IEEE Transactions on Electronic Computers}, EC-14\penalty0 (3):\penalty0 326--334, 1965.
\newblock \doi{10.1109/PGEC.1965.264137}.

\bibitem[Dauphin and Cubuk(2021)]{dauphin2021deconstructing}
Y.~Dauphin and E.~D. Cubuk.
\newblock Deconstructing the regularization of batchnorm.
\newblock In \emph{International Conference on Learning Representations (ICLR)}, 2021.

\bibitem[Deng et~al.(2019)Deng, Kammoun, et~al.]{deng2019doubledescent}
Z.~Deng, A.~Kammoun, et~al.
\newblock A model of double descent for high-dimensional binary classification.
\newblock \emph{arXiv preprint arXiv:1912.00887}, 2019.

\bibitem[Fisher(1936)]{fisher1936use}
R.~A. Fisher.
\newblock The use of multiple measurements in taxonomic problems.
\newblock \emph{Annals of Eugenics}, 7:\penalty0 179--188, 1936.
\newblock \doi{10.1111/j.1469-1809.1936.tb02137.x}.

\bibitem[Gao et~al.(2020)Gao, Wang, Herold, Yang, and Ney]{gao2020towards}
Y.~Gao, W.~Wang, C.~Herold, Z.~Yang, and H.~Ney.
\newblock Towards a better understanding of label smoothing in neural machine translation.
\newblock In \emph{Proceedings of the 1st Conference of the Asia-Pacific Chapter of the Association for Computational Linguistics and the 10th International Joint Conference on Natural Language Processing}, pages 212--223, 2020.

\bibitem[Gardner(1988)]{gardner1988space}
E.~Gardner.
\newblock The space of interactions in neural network models.
\newblock \emph{Journal of Physics A: Mathematical and General}, 21\penalty0 (1):\penalty0 257--270, 1988.

\bibitem[Gardner and Derrida(1989)]{gardner1989unfinished}
E.~Gardner and B.~Derrida.
\newblock Three unfinished works on the optimal storage capacity of networks.
\newblock \emph{Journal of Physics A: Mathematical and General}, 22:\penalty0 1983--1994, 1989.

\bibitem[Garrod and Keating(2025)]{garrod2024persistence}
C.~Garrod and J.~P. Keating.
\newblock The persistence of neural collapse despite low-rank bias, 2025.
\newblock URL \url{https://arxiv.org/abs/2410.23169}.

\bibitem[Guo et~al.(2017)Guo, Pleiss, Sun, and Weinberger]{guo2017calibration}
C.~Guo, G.~Pleiss, Y.~Sun, and K.~Q. Weinberger.
\newblock On calibration of modern neural networks.
\newblock In \emph{Proceedings of the International Conference on Machine Learning (ICML)}, pages 1321--1330. PMLR, 2017.

\bibitem[Guo et~al.(2024)Guo, Ross, Zhao, Andriopoulos, Ling, Xu, and Dong]{guo2024cross}
L.~Guo, K.~Ross, Z.~Zhao, G.~Andriopoulos, S.~Ling, Y.~Xu, and Z.~Dong.
\newblock Cross entropy versus label smoothing: A neural collapse perspective.
\newblock \emph{arXiv preprint arXiv:2402.03979}, 2024.

\bibitem[Han et~al.(2022)Han, Papyan, and Donoho]{han2022neuralcollapsemse}
X.~Y. Han, V.~Papyan, and D.~L. Donoho.
\newblock Neural collapse under mse loss: Proximity to and dynamics on the simplex etf.
\newblock In \emph{International Conference on Learning Representations (ICLR)}, 2022.
\newblock URL \url{https://openreview.net/forum?id=w1UbdvWH_R3}.

\bibitem[He et~al.(2015)He, Zhang, Ren, and Sun]{resnet18}
K.~He, X.~Zhang, S.~Ren, and J.~Sun.
\newblock Deep residual learning for image recognition, 2015.
\newblock URL \url{https://arxiv.org/abs/1512.03385}.

\bibitem[Hendrycks and Dietterich(2019)]{hendrycks2019benchmarking}
D.~Hendrycks and T.~Dietterich.
\newblock Benchmarking neural network robustness to common corruptions and perturbations.
\newblock In \emph{International Conference on Learning Representations (ICLR)}, 2019.
\newblock URL \url{https://arxiv.org/abs/1903.12261}.

\bibitem[Hinton et~al.(2015)Hinton, Vinyals, and Dean]{hinton2015distilling}
G.~Hinton, O.~Vinyals, and J.~Dean.
\newblock Distilling the knowledge in a neural network.
\newblock \emph{arXiv preprint arXiv:1503.02531}, 2015.
\newblock URL \url{https://arxiv.org/abs/1503.02531}.

\bibitem[Ji and Telgarsky(2019)]{ji2019implicit}
Z.~Ji and M.~Telgarsky.
\newblock The implicit bias of gradient descent on nonseparable data.
\newblock In \emph{Proceedings of The 32nd Conference on Learning Theory (COLT)}, 2019.
\newblock URL \url{https://proceedings.mlr.press/v99/ji19a.html}.

\bibitem[Kornblith et~al.(2021)Kornblith, Chen, Lee, and Norouzi]{kornblith2021better}
S.~Kornblith, T.~Chen, H.~Lee, and M.~Norouzi.
\newblock Why do better loss functions lead to less transferable features?
\newblock \emph{Advances in Neural Information Processing Systems (NeurIPS)}, 34:\penalty0 28648--28662, 2021.

\bibitem[Krizhevsky(2009)]{krizhevsky2009cifar}
A.~Krizhevsky.
\newblock Learning multiple layers of features from tiny images.
\newblock Technical report, University of Toronto, 2009.

\bibitem[Lee et~al.(2022)Lee, Cheung, and Zhang]{lee2022adaptive}
D.~Lee, K.~C. Cheung, and N.~L. Zhang.
\newblock Adaptive label smoothing with self-knowledge in natural language generation.
\newblock \emph{arXiv preprint arXiv:2210.13459}, 2022.

\bibitem[Levi et~al.(2024)Levi, Beck, and Bar-Sinai]{levi2024grokkinglinear}
N.~I. Levi, A.~Beck, and Y.~Bar-Sinai.
\newblock Grokking in linear estimators -- a solvable model that groks without understanding.
\newblock In \emph{International Conference on Learning Representations (ICLR)}, 2024.
\newblock URL \url{https://openreview.net/forum?id=GH2LYb9XV0}.

\bibitem[Liang et~al.(2022)Liang, Li, Bing, Zhao, Tang, Lin, and Fan]{liang2022efficient}
J.~Liang, L.~Li, Z.~Bing, B.~Zhao, Y.~Tang, B.~Lin, and H.~Fan.
\newblock Efficient one pass self-distillation with zipf’s label smoothing.
\newblock In \emph{European Conference on Computer Vision (ECCV)}, pages 104--119. Springer, 2022.

\bibitem[Lyu and Li(2020)]{lyu2020margin}
K.~Lyu and J.~Li.
\newblock Gradient descent maximizes the margin of homogeneous neural networks.
\newblock In \emph{International Conference on Learning Representations (ICLR)}, 2020.
\newblock URL \url{https://arxiv.org/abs/1906.05890}.

\bibitem[Mignacco et~al.(2020)Mignacco, Krzakala, Zdeborov{\'a}, et~al.]{mignacco2020regularization}
F.~Mignacco, F.~Krzakala, L.~Zdeborov{\'a}, et~al.
\newblock The role of regularization in classification of high-dimensional noisy gaussian mixture.
\newblock In \emph{Proceedings of the 37th International Conference on Machine Learning (ICML)}, 2020.

\bibitem[Montanari et~al.(2019)Montanari, Ruan, Sohn, and Yan]{montanari2019maxmargin}
A.~Montanari, F.~Ruan, Y.~Sohn, and J.~Yan.
\newblock The generalization error of max-margin linear classifiers: High-dimensional asymptotics in the overparameterized regime.
\newblock \emph{arXiv preprint arXiv:1911.01544}, 2019.
\newblock URL \url{https://arxiv.org/abs/1911.01544}.

\bibitem[M{\"u}ller et~al.(2019)M{\"u}ller, Kornblith, and Hinton]{muller2019when}
R.~M{\"u}ller, S.~Kornblith, and G.~E. Hinton.
\newblock When does label smoothing help?
\newblock In \emph{Advances in Neural Information Processing Systems (NeurIPS)}, 2019.
\newblock URL \url{https://arxiv.org/abs/1906.02629}.

\bibitem[Papyan et~al.(2020)Papyan, Han, and Donoho]{papyan2020neuralcollapse}
V.~Papyan, X.~Han, and D.~L. Donoho.
\newblock Prevalence of neural collapse during the terminal phase of deep learning training.
\newblock \emph{Proceedings of the National Academy of Sciences}, 2020.
\newblock URL \url{https://arxiv.org/abs/2008.08186}.

\bibitem[Pereyra et~al.(2017)Pereyra, Tucker, Chorowski, Kaiser, and Hinton]{pereyra2017regularizing}
G.~Pereyra, G.~Tucker, J.~Chorowski, {\L}.~Kaiser, and G.~Hinton.
\newblock Regularizing neural networks by penalizing confident output distributions.
\newblock \emph{arXiv preprint arXiv:1701.06548}, 2017.

\bibitem[Power et~al.(2022)Power, Burda, Edwards, Babuschkin, and Misra]{power2022grokking}
A.~Power, Y.~Burda, H.~Edwards, I.~Babuschkin, and V.~Misra.
\newblock Grokking: Generalization beyond overfitting on small algorithmic datasets.
\newblock \emph{arXiv preprint arXiv:2201.02177}, 2022.
\newblock URL \url{https://arxiv.org/abs/2201.02177}.

\bibitem[Prieto et~al.(2025)Prieto, Barsbey, Mediano, and Birdal]{prieto2025grokkingedgenumericalstability}
L.~Prieto, M.~Barsbey, P.~A.~M. Mediano, and T.~Birdal.
\newblock Grokking at the edge of numerical stability, 2025.
\newblock URL \url{https://arxiv.org/abs/2501.04697}.

\bibitem[Shen et~al.(2021)Shen, Liu, Xu, Chen, Cheng, and Savvides]{shen2021label}
Z.~Shen, Z.~Liu, D.~Xu, Z.~Chen, K.-T. Cheng, and M.~Savvides.
\newblock Is label smoothing truly incompatible with knowledge distillation: An empirical study.
\newblock \emph{arXiv preprint arXiv:2104.00676}, 2021.

\bibitem[Soudry et~al.(2018)Soudry, Hoffer, Nacson, Gunasekar, and Srebro]{soudry2018implicit}
D.~Soudry, E.~Hoffer, M.~S. Nacson, S.~Gunasekar, and N.~Srebro.
\newblock The implicit bias of gradient descent on separable data.
\newblock \emph{Journal of Machine Learning Research}, 19\penalty0 (70):\penalty0 1--57, 2018.
\newblock URL \url{https://arxiv.org/abs/1710.10345}.

\bibitem[S{\'u}ken{\'\i}k et~al.(2024)S{\'u}ken{\'\i}k, Mondelli, and Lampert]{sukenik2024neural}
P.~S{\'u}ken{\'\i}k, M.~Mondelli, and C.~Lampert.
\newblock Neural collapse versus low-rank bias: Is deep neural collapse really optimal?
\newblock \emph{arXiv preprint arXiv:2405.14468}, 2024.

\bibitem[Sur and Cand{\`e}s(2019)]{sur2019modern}
P.~Sur and E.~J. Cand{\`e}s.
\newblock A modern maximum-likelihood theory for high-dimensional logistic regression.
\newblock In \emph{Proceedings of the National Academy of Sciences}, 2019.
\newblock URL \url{https://arxiv.org/abs/1701.05023}.

\bibitem[Szegedy et~al.(2016)Szegedy, Vanhoucke, Ioffe, Shlens, and Wojna]{szegedy2016rethinking}
C.~Szegedy, V.~Vanhoucke, S.~Ioffe, J.~Shlens, and Z.~Wojna.
\newblock Rethinking the inception architecture for computer vision.
\newblock In \emph{Proceedings of the IEEE/CVF Conference on Computer Vision and Pattern Recognition (CVPR)}, pages 2818--2826, 2016.

\bibitem[Tikhonov and Arsenin(1977)]{Tikhonov1977SolutionsOI}
A.~N. Tikhonov and V.~Y. Arsenin.
\newblock Solutions of ill-posed problems.
\newblock 1977.
\newblock URL \url{https://api.semanticscholar.org/CorpusID:122072756}.

\bibitem[Wang and Thrampoulidis(2022)]{wang2022gaussianmixtures}
K.~Wang and C.~Thrampoulidis.
\newblock Binary classification of gaussian mixtures: Abundance of support vectors, robustness, and generalization.
\newblock \emph{SIAM Journal on Mathematics of Data Science}, 2022.
\newblock URL \url{https://arxiv.org/abs/2010.00000}.

\bibitem[Wei et~al.(2022)Wei, Xie, Cheng, Feng, An, and Li]{wei2022mitigatingneuralnetworkoverconfidence}
H.~Wei, R.~Xie, H.~Cheng, L.~Feng, B.~An, and Y.~Li.
\newblock Mitigating neural network overconfidence with logit normalization, 2022.
\newblock URL \url{https://arxiv.org/abs/2205.09310}.

\bibitem[Xia et~al.(2025)Xia, Laurent, Franchi, and Bouganis]{xiatowards}
G.~Xia, O.~Laurent, G.~Franchi, and C.-S. Bouganis.
\newblock Towards understanding why label smoothing degrades selective classification and how to fix it.
\newblock In \emph{International Conference on Learning Representations (ICLR)}, 2025.

\bibitem[Xu and Liu(2023)]{xu2023quantifying}
J.~Xu and H.~Liu.
\newblock Quantifying the variability collapse of neural networks.
\newblock In \emph{Proceedings of the International Conference on Machine Learning (ICML)}, pages 38535--38550. PMLR, 2023.

\bibitem[Yuan et~al.(2020)Yuan, Tay, Li, Wang, and Feng]{yuan2020revisiting}
L.~Yuan, F.~E. Tay, G.~Li, T.~Wang, and J.~Feng.
\newblock Revisiting knowledge distillation via label smoothing regularization.
\newblock In \emph{Proceedings of the IEEE/CVF Conference on Computer Vision and Pattern Recognition (CVPR)}, pages 3903--3911, 2020.

\bibitem[Zhang et~al.(2021)Zhang, Jiang, Hou, Wei, Han, Li, and Cheng]{zhang2021delving}
C.-B. Zhang, P.-T. Jiang, Q.~Hou, Y.~Wei, Q.~Han, Z.~Li, and M.-M. Cheng.
\newblock Delving deep into label smoothing.
\newblock \emph{IEEE Transactions on Image Processing}, 30:\penalty0 5984--5996, 2021.

\bibitem[Zhou et~al.(2022)Zhou, You, Li, Liu, Liu, Qu, and Zhu]{zhou2022all}
J.~Zhou, C.~You, X.~Li, K.~Liu, S.~Liu, Q.~Qu, and Z.~Zhu.
\newblock Are all losses created equal: A neural collapse perspective.
\newblock \emph{Advances in Neural Information Processing Systems (NeurIPS)}, 35:\penalty0 31697--31710, 2022.

\bibitem[Zhu et~al.(2022)Zhu, Cheng, Zhang, and Liu]{zhu2022rethinking}
F.~Zhu, Z.~Cheng, X.-Y. Zhang, and C.-L. Liu.
\newblock Rethinking confidence calibration for failure prediction.
\newblock In \emph{European Conference on Computer Vision (ECCV)}, pages 518--536. Springer, 2022.

\end{thebibliography}
\bibliographystyle{abbrvnat}

\appendix
\onecolumn
\part*{\clearpage Appendices}

\section{Additional Related Work}
\label[appendix]{appendix:related_work}

\textbf{Label smoothing and soft-target training.}
Label smoothing (LS) was introduced as ``label-smoothing regularization'' by \citet{szegedy2016rethinking} and is now a standard tool for improving generalization and calibration in classification and sequence models \citep{muller2019when,guo2017calibration}. LS has been analyzed across domains, including neural machine translation \citep{gao2020towards}. More broadly, LS is one instance of training with \emph{soft targets}, closely related to knowledge distillation \citep{hinton2015distilling, yuan2020revisiting}. The interaction between LS and distillation is nuanced, with studies investigating when LS-trained teachers help or harm KD \citep{shen2021label,chandrasegaran2022revisiting}. Beyond fixed uniform smoothing, variants adapt the target distribution during training, including Online Label Smoothing \citep{zhang2021delving}, Zipf label smoothing \citep{liang2022efficient}, and adaptive/self-knowledge smoothing \citep{lee2022adaptive}. LS has also been linked to changes in feature geometry \citep{muller2019when,kornblith2021better} and quantified via feature-variability measures \citep{xu2023quantifying}. However, LS can be problematic for selective classification or failure prediction, motivating targeted fixes \citep{zhu2022rethinking,xiatowards}.

\textbf{Logit/output regularization beyond classic label smoothing.}
Methods that regularize the \emph{logits} directly include entropy penalties \citep{pereyra2017regularizing}, logit-norm regularizers \citep{dauphin2021deconstructing}, and logit normalization \citep{wei2022mitigatingneuralnetworkoverconfidence}. A key theme is whether the loss remains \emph{lower bounded} and admits a \emph{finite} per-sample optimum, as in convex logit penalties, versus cross-entropy whose infimum is achieved only at infinite margin.

\textbf{Implicit bias of cross-entropy and margin maximization.}
For separable linear classification, gradient descent on unregularized cross-entropy converges to the hard-margin SVM \citep{soudry2018implicit}. This has been extended to nonseparable data \citep{ji2019implicit} and overparameterized homogeneous models \citep{lyu2020margin}. Convex logit regularization provides a sharp contrast by forcing the per-sample objective to attain a finite minimum.

\textbf{High-dimensional solvable models.}
Our analysis connects to classical work on the combinatorial ``capacity'' of linear separators \citep{cover1965geometrical, gardner1988space}. Recent work develops precise high-dimensional asymptotics for max-margin classifiers via CGMT/AMP techniques \citep{montanari2019maxmargin, deng2019doubledescent, mignacco2020regularization, wang2022gaussianmixtures}. In related directions, high-dimensional logistic regression exhibits phase-transition phenomena tied to the existence of maximum-likelihood estimators \citep{sur2019modern, celentano2021fom}.

\textbf{Grokking and delayed generalization.}
Grokking was popularized in neural networks trained on small algorithmic datasets \citep{power2022grokking}. Solvable models have since been developed for linear estimators \citep{levi2024grokkinglinear} and logistic regression near the edge of separability \citep{beck2025grokking}. Logit regularization in this context was briefly discussed in \citet{prieto2025grokkingedgenumericalstability}; we identify an analogous delayed generalization region with a shifted threshold.

\textbf{Simplex geometry and neural collapse.}
Our multi-class setup relates to the \emph{neural collapse} phenomenon \citep{papyan2020neuralcollapse, han2022neuralcollapsemse}, where logits converge to structured simplex/ETF configurations. Recent work studies how different losses shape this collapse \citep{zhou2022all, guo2024cross} and how it interacts with low-rank bias \citep{sukenik2024neural, garrod2024persistence}.


\section{Label smoothing as a special case of logit regularization}
\label[appendix]{appendix:label_smoothing}

We show that label smoothing (LS) can be written as a convex \emph{logit regularizer} added to cross-entropy, matching our form
$\ell = (1-\alpha)\ell_{\mathrm{CE}}+\alpha f$ (up to a reparameterization of $\alpha$ and additive constants).

\subsection*{Binary classification}

Consider a two-class softmax. Let $t \equiv yz$ be the signed logit, so that
$p_y=\sigma(t)=1/(1+e^{-t})$ and $p_{-y}=1-p_y=\sigma(-t)$. Standard cross-entropy is
$\ell_{\mathrm{CE}}(t)=-\log p_y=\log(1+e^{-t})$.
Binary LS with smoothing $\varepsilon$ uses targets $(1-\varepsilon,\varepsilon)$ and yields
\begin{equation}
\ell_{\mathrm{LS}}(t) = -(1-\varepsilon)\log p_y - \varepsilon \log p_{-y}
= (1-\varepsilon)\log(1+e^{-t})+\varepsilon\log(1+e^{t}).
\end{equation}
Using $\log(1+e^{t})+\log(1+e^{-t})=\log\!\big(2+2\cosh t\big)$,
we rewrite
\begin{equation}
\ell_{\mathrm{LS}}(t)=(1-2\varepsilon)\,\ell_{\mathrm{CE}}(t)\;+\;\varepsilon\,\log\!\big(2+2\cosh t\big).
\end{equation}
Thus, in the binary case LS is equivalent to logit regularization with
\begin{equation}
f_{\mathrm{LS}}(t)=\frac{1}{2}\log\!\big(2+2\cosh t\big),\qquad t=yz,
\end{equation}
(up to rescaling $\alpha=2\varepsilon$).

\subsection*{Multiclass classification}

Let $p_k(\bz)=\exp(z_k)/\sum_j \exp(z_j)$ and one-hot labels $y_k\in\{0,1\}$. LS replaces
$y_k^{\mathrm{LS}}=(1-\varepsilon)y_k+\varepsilon/K$, giving
\begin{align}
\mathcal{L}_{\mathrm{LS}}(\bz,\by)
&= -\sum_{k=1}^K y_k^{\mathrm{LS}}\log p_k(\bz) \nonumber\\
&= (1-\varepsilon)\,\mathcal{L}_{\mathrm{CE}}(\bz,\by)\;+\;\frac{\varepsilon}{K}\sum_{k=1}^K \bigl[-\log p_k(\bz)\bigr].
\end{align}
Hence LS is of the form $(1-\alpha)\mathcal{L}_{\mathrm{CE}}+\alpha f(\bz)$ with the convex, permutation-symmetric logit regularizer
\begin{equation}
f_{\mathrm{LS}}(\bz)=\log\!\Big(\sum_{k=1}^K e^{z_k}\Big)\;-\;\frac{1}{K}\sum_{k=1}^K z_k,
\end{equation}

\section{Complementary details for \cref{sec:alpha_independence}}

In \cref{sec:alpha_independence} we showed that under some assumptions, the direction of the weight vector becomes almost invariant to the logit-regularization function in general, and in particular to the regularization strength $\alpha$. We did by proving \cref{prop:gaussian_data} and \cref{cor:quadratic_loss}, where we showed if the distribution of the data is strictly gaussian, this will be true for any convex per-sample loss, and on the other hand if the per-sample loss is quadratic, then it will be true regardless of its minimum.

\subsection{Complementary proofs}
\label[appendix]{app:complementary_proofs}

We begin by proving a Lemma which is needed for the proof of  \cref{prop:gaussian_data}.

\begin{lemma}\label{lemma:monotonicity_of_std}
Let $x$ be a zero-mean random variable and let $f:\mathbb {R}\to\mathbb{R}$ be a convex function. Then $\avg{f(\lambda x)}$ is an nondecreasing function of $\lambda$ for $\lambda >0$. 
\end{lemma}
\begin{proof}
    This is a standard convexity argument. Let's examine two lambdas $0<\lambda_1<\lambda_2$. Then $\lambda_1$ is a convex combination of 0 and $\lambda_2$, 
    $$\lambda_1 = \left(1-\frac{\lambda_1}{\lambda_2}\right)\cdot 0 + \frac{\lambda_1}{\lambda_2}\lambda_2\ .$$
    Therefore,
    \begin{align*}
        \avg{f\left(\lambda_1 x\right)}&=
        \avg{f\left(\left(1-\frac{\lambda_1}{\lambda_2}\right)\cdot 0 \cdot x + \frac{\lambda_1}{\lambda_2}\lambda_2 x\right)} \\
        &\le \left(1-\frac{\lambda_1}{\lambda_2}\right) f(0) + \frac{\lambda_1}{\lambda_2}\avg{f(\lambda_2 x)}\\
        &\le \avg{\left(1-\frac{\lambda_1}{\lambda_2}\right) f(\lambda_2 x) + \frac{\lambda_1}{\lambda_2}f(\lambda_2 x)}
        =\avg{f(\lambda_2 x)}
    \end{align*}
    Where we used the fact that $x$ is zero-mean and therefore $f(0)=f(\avg{ \lambda_2 x})\le \avg{f(\lambda_2 x)}$.
\end{proof}

\textbf{The connection to \cref{prop:gaussian_data}:} 
For fixed $\mu\in\mathbb{R}$, define $g(t)\equiv f(\mu+t)$, which is convex whenever $f$ is convex. Applying Lemma~\ref{lemma:monotonicity_of_std} with $x=Y\sim\mathcal{N}(0,1)$ yields that
\[
Q(\mu,\sigma)\;=\;\mathbb{E}\bigl[f(\mu+\sigma Y)\bigr]\;=\;\mathbb{E}\bigl[g(\sigma Y)\bigr]
\]
is nondecreasing in $\sigma\ge 0$ for each fixed $\mu$.

We will now provide the proof to \cref{cor:directional_invariance} in the main text.
\begin{theorem}[Restatement of \cref{cor:directional_invariance}]
    If the data is Gaussian, the optimal weight direction is $\hat{\bs}_{\min}\propto \Sigma^{-1} \mu$ and therefore independent of the convex regularization function $f$. Similarly, if  $\lambda=d/N<1$ and the per-sample loss is quadratic, $\hat{\bs}_{\min} \propto \Sigma^{-1} \mu$ (where $\Sigma,\mu$ are the empirical mean and centered covariance) and therefore independent of the target value $z^*$.
\end{theorem}

\begin{proof}
\textbf{Population Gaussian.}
For $\bx\sim\mathcal{N}(\boldsymbol{\mu},\Sigma)$ and any $\boldsymbol{S}$, the logit
$z=\boldsymbol{S}^\top\bx$ is Gaussian with mean $\mu(\boldsymbol{S})=\boldsymbol{S}^\top\boldsymbol{\mu}$
and variance $\sigma^2(\boldsymbol{S})=\boldsymbol{S}^\top\Sigma\boldsymbol{S}$.
By Proposition~\ref{prop:gaussian_data}, any minimizer of $L(\boldsymbol{S})=\mathbb{E}[\ell(z)]$
minimizes the ratio $r(\boldsymbol{S})=\sigma(\boldsymbol{S})/\mu(\boldsymbol{S})$ over $\{\mu(\boldsymbol{S})>0\}$, i.e.
\[
r(\boldsymbol{S})^2=\frac{\boldsymbol{S}^\top\Sigma\boldsymbol{S}}{(\boldsymbol{S}^\top\boldsymbol{\mu})^2}.
\]
Using Cauchy--Schwarz,
\[
(\boldsymbol{S}^\top\boldsymbol{\mu})^2
=\langle \Sigma^{1/2}\boldsymbol{S},\Sigma^{-1/2}\boldsymbol{\mu}\rangle^2
\le (\boldsymbol{S}^\top\Sigma\boldsymbol{S})\,(\boldsymbol{\mu}^\top\Sigma^{-1}\boldsymbol{\mu}),
\]
hence $r(\boldsymbol{S})^2\ge (\boldsymbol{\mu}^\top\Sigma^{-1}\boldsymbol{\mu})^{-1}$.
Moreover, equality holds iff $\Sigma^{1/2}\boldsymbol{S}\propto  \Sigma^{-1/2}\boldsymbol{\mu}$,
equivalently $\boldsymbol{S}\propto \Sigma^{-1}\boldsymbol{\mu}$.

\textbf{Empirical quadratic.}
Let $\ell(z)=a(z-z^*)^2$ and define $\hat{\boldsymbol{\mu}}=\frac1N\sum_i\bx_i$,
$\hat{\Sigma}=\frac1N\sum_i(\bx_i-\hat{\boldsymbol{\mu}})(\bx_i-\hat{\boldsymbol{\mu}})^\top$.
Assume that $\Sigma$ is invertible (since $\lambda<1$),
then
\[
\hat L(\boldsymbol{S})=\frac{a}{N}\sum_i(\boldsymbol{S}^\top\bx_i-z^*)^2
=a\,\boldsymbol{S}^\top\hat{\Sigma}\boldsymbol{S}+a(\boldsymbol{S}^\top\hat{\boldsymbol{\mu}}-z^*)^2,
\]
so the optimal \emph{direction} minimizes
$\hat r(\boldsymbol{S})^2=\frac{\boldsymbol{S}^\top\hat{\Sigma}\boldsymbol{S}}{(\boldsymbol{S}^\top\hat{\boldsymbol{\mu}})^2}$ and is independent of $z^*$.
Repeating the Cauchy--Schwarz step from the Gaussian case gives the result.

\end{proof}

\subsection{Additional numerical evidence}
\label[appendix]{app:additional_numerical_evidence}

We numerically illustrate the finite-sample mechanism behind the $\alpha$-plateau discussed in \cref{sec:alpha_independence}.
For each $N$, we sample $x_i\in\mathbb{R}^2$ i.i.d.\ from a Gaussian and compute the empirical minimizer
\[
S_{\min}(\alpha)=\arg\min_S \frac1N\sum_{i=1}^N\Big[(1-\alpha)\log(1+e^{-x_i^\top S})+\alpha(x_i^\top S)^2\Big].
\]
As a reference direction, we also compute $S_{\min}^{\mathrm{quad}}$ for the purely quadratic per-sample loss $b(z-a)^2$ on the same dataset (with values $b=a=1$, but we also verified that it does not matter what are these values from \cref{cor:quadratic_loss}), and plot
$\tan\theta(\alpha)$ where $\theta(\alpha)$ is the angle between $ S_{\min}(\alpha)$ and $ S_{\min}^{\mathrm{quad}}$.

\Cref{fig:finite_sample_alpha_plateau} shows that for small $N$, $ S_{\min}(\alpha)$ can deviate at very small $\alpha$ (finite-sample non-Gaussianity), but it stabilizes for large $\alpha$ as the objective becomes effectively quadratic. For larger $N$, the plateau is reached at much smaller $\alpha$, consistent with improved approximate Gaussianity.

\begin{figure}[h]
\centering
\includegraphics[width=0.6\linewidth]{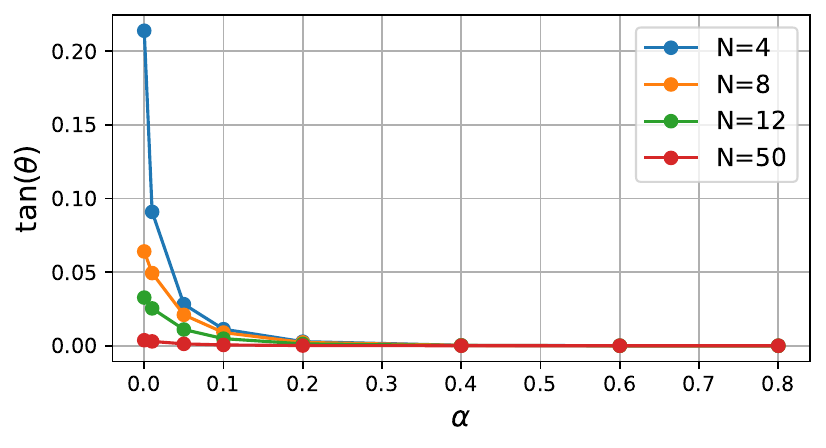} 
\caption{\textbf{Finite-sample $\alpha$-plateau.} $\tan\theta(\alpha)$ between $ S_{\min}(\alpha)$ for $(1-\alpha)\log(1+e^{-z})+\alpha z^2$ and the reference direction $ S_{\min}^{\mathrm{quad}}$ for $b(z-a)^2$, for several $N$. Small $N$ shows deviations at small $\alpha$. Increasing $N$ makes the plateau appear at even smaller $\alpha$. Note that the smallest $\alpha$ in this figure is $\alpha=0.00001$.}
\label{fig:finite_sample_alpha_plateau}
\end{figure}

\textbf{Two toy examples with $N=3$ points.}
We next give two minimal $N=3$ demonstrations that separate the “quadratic invariance” mechanism from the “Gaussian-logit” mechanism.

\textbf{Figure~\ref{fig:toy_quadratic_shift_a}: changing $a$ only (quadratic case).}
We sample three points $x_1,x_2,x_3\in\mathbb{R}^2$ in random position and minimize
\[
\frac{1}{3}\sum_{i=1}^3 \ell(z_i),
\qquad
z_i=S^\top x_i,
\qquad
\ell(z)=(z-a)^2 + b\,\max(0,-z).
\]
In the first toy example we set $b=0$, so the objective is purely quadratic in $z$.
We compare two runs with $(a,b)=(2,0)$ and $(a,b)=(6,0)$.
As shown in \Cref{fig:toy_quadratic_shift_a}, the minimizers $S^{\min}_1$ and $S^{\min}_2$ point in exactly the same direction and differ only by a scalar factor (left panel).
Accordingly, the three logits $z_i=S^\top x_i$ shift to different absolute values between the two runs (right panel), but their “clustering ratio”
\[
r \;=\; \frac{\sigma}{\mu}
\]
(where $\mu$ and $\sigma$ are the empirical mean and standard deviation over the three logits) remains constant: we report $(\mu,\sigma,r)$ in the table in the bottom row of the figure.
This is a direct finite-sample illustration of the direction-independence in the quadratic regime (\cref{cor:quadratic_loss}): changing the quadratic target $a$ rescales the minimizer without changing its direction.

\begin{figure}[t]
\centering
\includegraphics[width=\linewidth]{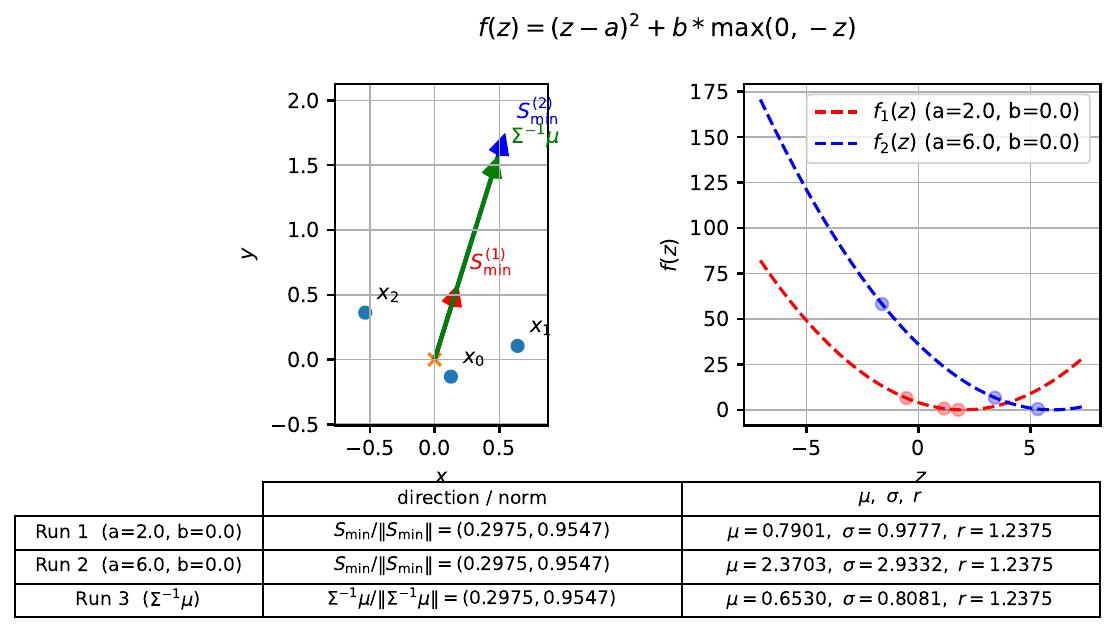} 
\caption{\emph{Toy example ($N=3$) with a quadratic loss: changing $a$ rescales the minimizer but does not change its direction.}
We minimize $\frac13\sum_{i=1}^3 (S^\top x_i-a)^2$ over $S\in\mathbb{R}^2$ for two choices of $a$ (here $a=2$ and $a=6$, with $b=0$). We also present for comparison the direction $\Sigma^{-1} \bmu$.
\textbf{Left:} the three points $x_1,x_2,x_3$ and the two minimizers $S^{\min}_1,S^{\min}_2$, which are colinear.
\textbf{Right:} the corresponding logits $\{S^\top x_i\}_{i=1}^3$ for both runs (shown in different colors), which shift in absolute value.
\textbf{Bottom:} empirical $\mu$ (mean), $\sigma$ (std), and $r=\sigma/\mu$, showing that the coefficient of variation is unchanged across the two runs.}
\label{fig:toy_quadratic_shift_a}
\end{figure}

\textbf{Figure~\ref{fig:toy_nonlinear_b}: adding a hinge-like term (non-quadratic, non-Gaussian).}
In the second toy example we keep $a=2$ fixed and compare $(a,b)=(2,0)$ versus $(a,b)=(2,6)$.
Now the objective includes the non-quadratic term $b\max(0,-z)$, and with only three points the induced logit distribution is extremely far from Gaussian.
In this setting, neither of the sufficient conditions underlying our invariance arguments holds (quadratic loss or Gaussian logit distribution), and we indeed observe that the two minimizers are no longer perfectly colinear (though their angle remains moderate); see \Cref{fig:toy_nonlinear_b}.

\begin{figure}[t]
\centering
\includegraphics[width=\linewidth]{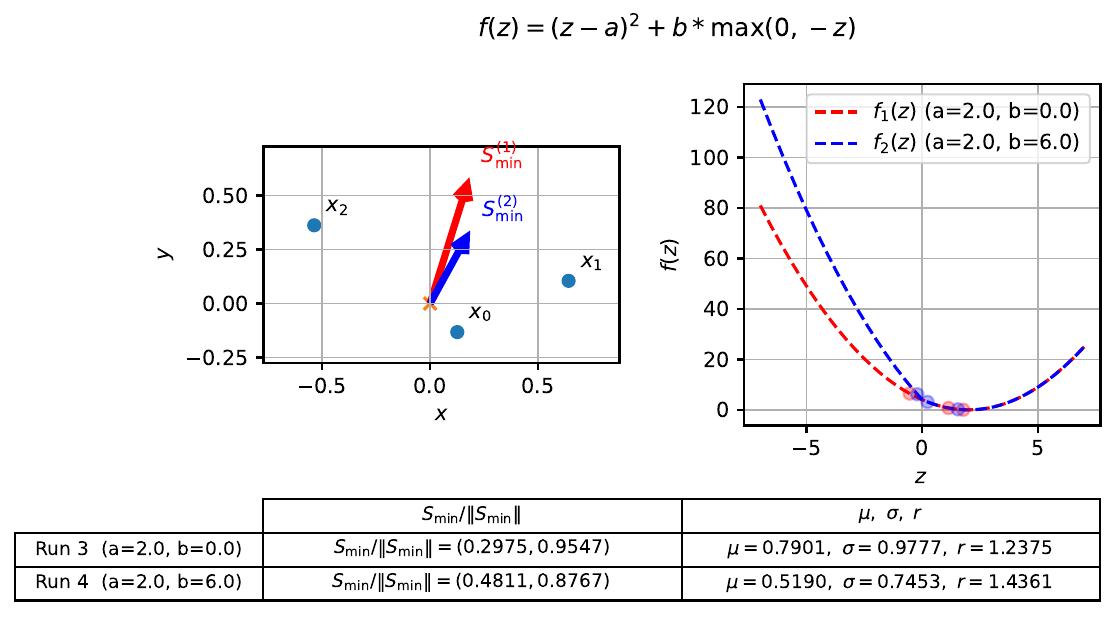} 
\caption{\emph{Toy example ($N=3$) where invariance need not hold: adding a hinge-like term changes the direction.}
We minimize $\frac13\sum_{i=1}^3 \big[(S^\top x_i-a)^2 + b\max(0,-S^\top x_i)\big]$ for two settings: $(a,b)=(2,0)$ and $(a,b)=(2,6)$ on the same three points.
Because the loss is no longer purely quadratic and the empirical logit distribution (with $N=3$) is far from Gaussian, the two minimizers need not share the same direction, and we observe a noticeable (but not extreme) angular deviation between them.}
\label{fig:toy_nonlinear_b}
\end{figure}

\textbf{The invariance to $\alpha$ in the case of the Gaussian model.}
In \cref{fig:alpha_independence_2} we show that for when the data is taken from Gaussian distribution, the direction is almost flat as a function of $\alpha$, even for small values. We show it for different values of $\lambda=d/N$. The slight deviation from the plateau is expected from finite-sampling effects.

\begin{figure}[h!]
\begin{centering}
\includegraphics[scale=0.8]{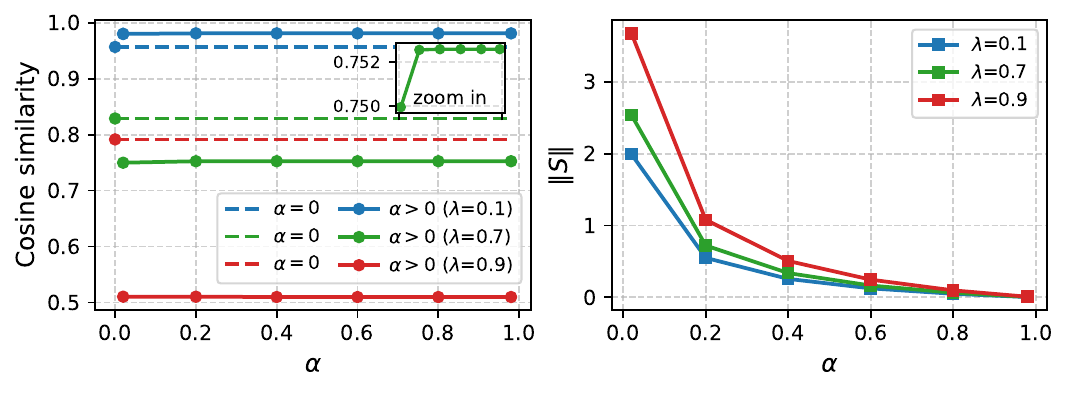}
\par\end{centering}
\caption{\label{fig:alpha_independence_2}\textbf{Left panel:} Cosine similarity of $\bs_{\min}$ with the x-axis. The dashed line indicates the unregularized value ($\alpha=0$). We observe an abrupt shift from $\alpha=0$ to $\alpha>0$, followed by a stable plateau, indicating that the weight direction is effectively independent of $\alpha$. \textbf{Right panel:} The norm $\|\bs_{\min}\|$, which, in contrast, exhibits a clear dependence on $\alpha$. }
\end{figure}

\subsection{Alternative proof of \cref{prop:gaussian_data} via monotonicity of $Q(r)$}
\label[appendix]{app:alt_proof_Qr}

We provide an equivalent proof based on defining an objective $Q(r)$ directly and showing it is strictly increasing in the coefficient of variation $r$.

\begin{proposition}[Restatement of \cref{prop:gaussian_data} for convenience]
\label{prop:gaussian_data_altproof}

 Let $\bx\in\mathbb{R}^{d}$ be Gaussian,
$\bx\sim\mathcal{N}(\bmu_{x},\Sigma_{x})$ with $\bmu_{x}\neq \mathbf{0}$.
Let $f:\mathbb{R}\to\mathbb{R}$ be convex, twice differentiable, and with a unique minimizer at $m>0$.
For $\bS\in\mathbb{R}^{d}$ define the logit
\[
z=\bS^{\top}\bx\sim\mathcal{N}(\mu,\sigma^{2}),
\qquad
\mu(\bS)=\bS^{\top}\bmu_{x},\quad
\sigma^{2}(\bS)=\bS^{\top}\Sigma_{x}\bS,
\]
and the population objective $L(\bS)=\mathbb{E}[f(z)]$.
Then any minimizer $\bS_{\min}\in\arg\min_{\bS}L(\bS)$ also minimizes the ratio
\[
r(\bS)\;=\;\frac{\sigma(\bS)}{\mu(\bS)}
\qquad \text{over all $\bS$ such that $\mu(\bS)>0$.}
\]
\end{proposition}

\begin{proof}
Write $\hat{\bS}=\bS/\|\bS\|$ and define
\[
\hat z \;=\; \hat{\bS}^{\top}\bx \sim \mathcal{N}(\hat\mu,\hat\sigma^{2}),
\qquad
\hat\mu=\hat{\bS}^{\top}\bmu_x,\quad
\hat\sigma^{2}=\hat{\bS}^{\top}\Sigma_x\hat{\bS}.
\]
We first note that any minimizing direction must satisfy $\hat\mu>0$.
Fix a unit direction $\hat{\bS}$ and consider the one-dimensional restriction
\[
h(t)\;=\;\mathbb{E}\!\left[f\!\left(t\,\hat z\right)\right],
\qquad t=\|\bS\|\ge 0.
\]
By convexity of $f$, $h$ is convex in $t$, and $h(0)=f(0)$ with
\[
h'(0)=\mathbb{E}\!\left[f'(0)\,\hat z\right]=\hat\mu\,f'(0).
\]
Since $f$ has a unique minimizer at $m>0$, convexity implies $f'(0)<0$.
If $\hat\mu<0$, then $h'(0)>0$ and by convexity $h$ is minimized at $t=0$.
If $\hat\mu>0$, then $h'(0)<0$ and $h$ attains its unique minimizer at some $t>0$.
Thus, flipping $\hat{\bS}\mapsto-\hat{\bS}$ strictly improves the loss near the origin whenever $\hat\mu<0$, so every minimizing direction must satisfy $\hat\mu>0$.

Assume henceforth $\mu(\bS)>0$. Let $Y\sim\mathcal{N}(0,1)$ so that
\[
z=\mu+\sigma Y=\mu(1+rY),\qquad r=\frac{\sigma}{\mu}.
\]
Define
\[
q(\mu,r)\;=\;\mathbb{E}\!\left[f\!\left(\mu(1+rY)\right)\right],
\qquad \mu>0,\ r\ge 0.
\]
For a fixed direction $\hat{\bS}$, scaling $\bS$ by a positive factor scales both $\mu$ and $\sigma$ by the same factor, hence leaves $r=\sigma/\mu$ invariant.
Therefore,
\[
L_{\min}
=\min_{\hat{\bS}}\min_{\mu>0} q(\mu,r(\hat{\bS})).
\]
Define the reduced objective
\[
Q(r)\;=\;\min_{\mu>0} q(\mu,r),
\]
so that
\[
L_{\min}=\min_{\hat{\bS}} Q(r(\hat{\bS})).
\]
It suffices to show $Q$ is strictly increasing on $(0,\infty)$.
Let $\mu^*(r)$ be a minimizer of $q(\mu,r)$ over $\mu>0$, so $Q(r)=q(\mu^*(r),r)$.
Differentiating and using $\partial_\mu q(\mu^*(r),r)=0$ at the optimum (envelope theorem), we obtain
\[
\frac{dQ}{dr}
=\frac{\partial q}{\partial r}(\mu^*(r),r)
=\mu^*(r)\,\mathbb{E}\!\left[f'\!\left(\mu^*(r)(1+rY)\right)Y\right].
\]
Using Gaussian integration by parts, $\mathbb{E}[g(Y)Y]=\mathbb{E}[g'(Y)]$ for $Y\sim\mathcal{N}(0,1)$, with
$g(y)=f'(\mu^*(1+ry))$ and $g'(y)=r\mu^* f''(\mu^*(1+ry))$, yields
\[
\mathbb{E}\!\left[f'\!\left(\mu^*(1+rY)\right)Y\right]
=r\mu^*\,\mathbb{E}\!\left[f''\!\left(\mu^*(1+rY)\right)\right].
\]
Hence
\[
\frac{dQ}{dr}
=r\,\mu^{*}(r)^{2}\,\mathbb{E}\!\left[f''\!\left(\mu^{*}(r)(1+rY)\right)\right].
\]
For $r>0$, $\mu^*(r)>0$, and convex $f$ we have $f''\ge 0$ (and $>0$ a.e.\ if $f$ is strictly convex), so $\frac{dQ}{dr}>0$ for all $r>0$.
Therefore $Q$ is increasing in $r$, and the minimizing direction must achieve the smallest attainable $r(\hat{\bS})=\hat\sigma/\hat\mu$.
Equivalently, any $\bS_{\min}\in\arg\min_{\bS}L(\bS)$ minimizes $r(\bS)=\sigma(\bS)/\mu(\bS)$ over $\mu(\bS)>0$.
\end{proof}

\subsection{Explicit Derivation for the Quadratic Case}
\label[appendix]{app:explicit_quadratic_derivation}

To build intuition, we consider the pedagogically illuminating case where the per-sample loss is purely quadratic, $\ell(z) = (z-a)^2$. In this setting, we can derive the optimal weight vector explicitly and demonstrate directly that minimizing the loss is equivalent to minimizing the coefficient of variation $r$.

Suppose that the loss is :
\begin{equation}
    \mathcal{L}(\boldsymbol{S}) =\mathbb{E}\left[ \left(z - a\right)^{2} \right],
\end{equation}
where $z=\bs^T x$.
Let the direction of the weights be $\hat{\boldsymbol{S}} = \boldsymbol{S}/\|\boldsymbol{S}\|$ and the norm be $g = \|\boldsymbol{S}\|$. We denote the unit-direction moments as $\hat{\mu} = \mathbb{E}[\hat{\boldsymbol{S}} \cdot \boldsymbol{x}]$ and $\hat{\sigma}^2 = \operatorname{Var}( \hat{\boldsymbol{S}} \cdot \boldsymbol{x})$. The statistics of the full logit are then $\mu = g\hat{\mu}$ and $\sigma = g\hat{\sigma}$. 

Using the bias-variance decomposition $\mathbb{E}[(z-a)^2] = \operatorname{Var}(z) + (\mathbb{E}[z]-a)^2$, the loss becomes:
\begin{equation}
    \mathcal{L}(g, \hat{\boldsymbol{S}}) \approx \sigma^2 + (\mu - a)^2 = g^2\hat{\sigma}^2 + (g\hat{\mu} - a)^2.
\end{equation}
We first optimize the scale $g$ for a fixed direction $\hat{\boldsymbol{S}}$. Taking the derivative with respect to $g$:
\begin{equation}
    \frac{\partial \mathcal{L}}{\partial g} = 2g\hat{\sigma}^2 + 2(g\hat{\mu} - a)\hat{\mu} = 0.
\end{equation}
Solving for $g$, we obtain the optimal norm:
\begin{equation}
    g^* = a \frac{\hat{\mu}}{\hat{\mu}^2 + \hat{\sigma}^2}.
\end{equation}
Substituting this back into the expression for the mean logit $\mu^* = g^* \hat{\mu}$, we find:
\begin{equation}
    \mu^* = a \frac{\hat{\mu}^2}{\hat{\mu}^2 + \hat{\sigma}^2} = \frac{a}{1 + r^2},
\end{equation}
where $r = \hat{\sigma}/\hat{\mu} = \sigma/\mu$ is the clustering ratio. This reveals a shrinkage effect: the optimal mean $\mu^*$ is slightly \emph{smaller} than the target $a$ to reduce the variance penalty.

Finally, substituting $g^*$ back into the loss function yields the minimal loss achievable for the direction $\hat{\boldsymbol{S}}$:
\begin{equation}
    \mathcal{L}_{\min}(\hat{\boldsymbol{S}}) 
    = g^{*2}\hat{\sigma}^2 + (g^*\hat{\mu} - a)^2 
    = a^2 \frac{\hat{\sigma}^2}{\hat{\mu}^2 + \hat{\sigma}^2} 
    = a^2 \frac{1}{1 + 1/r^2}.
\end{equation}
Therefore, minimizing the total loss is mathematically equivalent to minimizing the coefficient of variation $r$.

\section{Complementary details for \cref{sec:multiclass} (Multi-class classification)}
\label[appendix]{app:multiclass_additional_details}

In this section we will provide complementary details for the multiclass case. We will show that most of the results generalize naturally.

\subsection{Logit clustering}
\label[appendix]{app:multiclass_quadratic_invariance}
We consider a $K$-class linear classifier $\bz=W\bx+\bb$.
As discussed in the main text, when adding convex and symmetric logit-regularization function $f(\bz)$, the loss in the multi-class case can also be decomposed as a sum of per-sample loss functions.  The main difference is that now, each class has a different per-sample minimum. The goal of the optimizer is to cluster all of the logits of a certain class to its corresponding minimum, and hence the implicit bias would be, again, logit clustering.
From symmetry, this minimum should take the form
\begin{equation}
    \boldsymbol{z}_{c,k}^{*}=\begin{cases}
q_{\mathrm{high}} & k=c\\
q_{\mathrm{low}} & else
\end{cases}
\end{equation}
where $k=1,...,K$ is the class index, $c$ is the class of the logit (determined by the label), and $q_{\mathrm{high}}>q_{\mathrm{low}}$ are some scalars.
We will begin by considering a quadratic loss, and showing that the "direction" and optimal generalization accuracy is invariant to its specific details.

\subsection{Invariance to the Quadratic Logit Regularization Function}We generalize \cref{cor:quadratic_loss} to the multi-class setting. Assuming that the per-sample loss function is quadratic \emph{around the corresponding class minimum}, we show that generalization accuracy is invariant to the specific target values (and thus the regularization parameters).

\begin{proposition}[Quadratic Loss Invariance]\label{prop:multiclass_quadratic_invariance}Consider a multi-class linear classifier minimizing the empirical loss $\mathcal{L} = \frac{1}{N}\sum_i \ell(\bz_i, \by_i)$, where the per-sample loss is quadratic:\begin{equation}\ell(\bz, \by_c) = |\bz - \bz^*_c|^2,\end{equation}and $\bz^*_c\in \mathbb{R}^K$ is a target vector specific to class $c$. Then, for any $q \neq 0$, the generalization accuracy of the optimal solution is invariant to the value of $q$.\end{proposition}\begin{proof}Let $W \in \mathbb{R}^{K \times d}$ and $\bb \in \mathbb{R}^K$. The logits are given by $\bz_i = W\bx_i + \bb$. 
We notice that up to a global shift that can be compensated by the bias, we can assume that $q_{\mathrm{low}}=0$.
Therefore, $\bz^*_c = q \boldsymbol{e}_c$. The loss function is:\begin{equation}\mathcal{L}(W, \bb) = \frac{1}{N} \sum_{i=1}^N |W\bx_i + \bb - q \boldsymbol{y}i|^2,
\end{equation}
where $\boldsymbol{y}_i$ is the one-hot representation of the label.
We can factor out $q^2$:
\begin{equation}
\mathcal{L}(W, \bb) = q^2 \frac{1}{N} \sum{i=1}^N \left| \left(\frac{1}{q}W\right)\bx_i + \left(\frac{1}{q}\bb\right) - \boldsymbol{y}_i \right|^2.\end{equation}Let $\tilde{W} = W/q$ and $\tilde{\bb} = \bb/q$. Minimizing $\mathcal{L}$ with respect to $W, \bb$ is equivalent to minimizing the term inside the summation with respect to $\tilde{W}, \tilde{\bb}$. Let $\tilde{W}_{\min}, \tilde{\bb}_{\min}$ be the minimizers for the case $q=1$. Then the minimizers for an arbitrary $q$ are simply $W_{\min} = q \tilde{W}_{\min}$ and $\bb_{\min} = q \tilde{\bb}_{\min}$.The predicted class for a test point $\bx$ is:
\begin{equation}
    \hat{y} = \operatorname*{argmax}_{k} (W_{\min}\bx + \bb_{\min})_k = \operatorname*{argmax}_{k} (q(\tilde{W}_{\min}\bx + \tilde{\bb}_{\min}))_k.
\end{equation}
Since $q$ is a positive scalar, the argmax (and thus the accuracy) remains unchanged.
\end{proof}

In the binary setting (\cref{prop:gaussian_data}), the scalar coefficient of variation $r=\sigma/\mu$ is equivalent (up to inversion) to the LDA objective: maximizing $\mu/\sigma$ is the same as maximizing the Rayleigh quotient
\[
J(\bs)=\frac{(\bs^\top \bmu)^2}{\bs^\top \Sigma\,\bs},
\]
whose maximizer is $\bs\propto \Sigma^{-1}\bmu$. In the multi-class case, Fisher's criterion has a standard generalization obtained by replacing the scalar mean-variance ratio with a between-class versus within-class scatter ratio. Writing $\bmu_c$ for class means, $\bmu$ for the global mean, and defining the scatters
\begin{equation}
S_B=\sum_{c=1}^K \pi_c(\bmu_c-\bmu)(\bmu_c-\bmu)^\top,
\qquad
\end{equation}
the multi-class Fisher objective for a $(K\!-\!1)$-dimensional projection $U\in\mathbb{R}^{d\times (K-1)}$ is
\begin{equation}
J(U)=\mathrm{tr}\!\Big((U^\top S_W U)^{-1}\,U^\top S_B U\Big)
\end{equation}
(equivalently, $\det(U^\top S_B U)/\det(U^\top S_W U)$).
This reduces to the binary $\mu/\sigma$ criterion when $K=2$ and is maximized by the top $(K-1)$ generalized eigenvectors of the pair $(S_B,S_W)$ (i.e., of $S_W^{-1}S_B$). Accordingly, the natural multi-class analogue of our scalar clustering invariant is the Fisher scatter ratio encoded by these generalized eigenvalues (or their sum/product), rather than a single $\sigma/\mu$.

\subsection{Shift of the Interpolation Threshold ($\sigma_f=0$)}In this section, we provide the multi-class generalization for the shift of the interpolation threshold in the noiseless feature regime. We also demonstrate in \cref{fig:multiclass_grokking} that a grokking dynamics can be achieved, while also the fact that the logits cluster at a single point at the end of the training.

\begin{proposition}
\label{prop:multiclass_generalization_proof}Consider the multi-class model with $\alpha>0$ and $\sigma_{f}=0$. Assume the limit where $d, N \to \infty$ with $\lambda = d/N < 1$. Furthermore, assume that the class means $\{\boldsymbol{\mu}_{c}\}_{c=1}^{K} \subset \mathbb{R}^{d}$ span a $(K-1)$-dimensional affine subspace and are affinely independent. Then, the parameters $W_{\min}, \boldsymbol{b}_{\min}$ that minimize the training loss achieve perfect generalization accuracy.\end{proposition}\begin{proof}Since $\sigma_{f}=0$, any sample $\boldsymbol{x}^{(i)}$ belonging to class $c$ lies in a subspace defined by the class mean plus orthogonal noise:
\begin{equation}
\boldsymbol{x}^{(i)} = \boldsymbol{\mu}_{c} + \boldsymbol{\xi}^{(i)}_{\perp},
\end{equation}
where $\boldsymbol{\mu}_{c}$ is the deterministic signal component, and $\boldsymbol{\xi}^{(i)}_{\perp}$ lies in the orthogonal complement of the signal subspace $\mathcal{S} = \operatorname{span}(\{\boldsymbol{\mu}_k\})$.
The regularized loss is a sum of strictly convex per-sample losses. A global lower bound for the total loss is achieved if and only if the logit vector for every sample $i$ of class $c$ coincides exactly with the unique minimizer of the class-specific per-sample loss, $\bz^{*}_{c} \in \mathbb{R}^{K}$.

We decompose the weight matrix $W$ into a component acting on the signal subspace, $W_{\mathrm{sig}}$, and a component acting on the orthogonal noise, $W_{\perp}$. We construct a solution where $W_{\perp} = 0$. The condition for achieving the global minimum on the training set becomes:
\begin{equation}
    W_{\mathrm{sig}}\boldsymbol{\mu}_{c} + \boldsymbol{b} = \bz^{*}_{c}, \quad \forall c \in \{1, \dots, K\}.
\end{equation}
We can rewrite this using augmented notation. Let $\tilde{W}_{\mathrm{sig}} = [\boldsymbol{b}, W_{\mathrm{sig}}]$ and $\tilde{\boldsymbol{\mu}}_c = [1, \boldsymbol{\mu}_c^\top]^\top$. The system is:
\begin{equation}
    \tilde{W}_{\mathrm{sig}} M = Z^*,
\end{equation}
where columns of $M$ are the augmented means $\tilde{\boldsymbol{\mu}}_c$, and columns of $Z^*$ are the targets $\bz^*_c$. Since the class means are affinely independent, the matrix $M \in \mathbb{R}^{K \times K}$ is invertible. Thus, a unique solution $\tilde{W}_{\mathrm{sig}} = Z^* M^{-1}$ exists.

For $\lambda < 1$, the loss is strictly convex, ensuring this solution is unique. Thus, the learned weights map every sample (train or test) of class $c$ exactly to $\boldsymbol{z}^{*}_{c}$. As $\boldsymbol{z}^{*}_{c}$ minimizes the loss for label $c$, it must correctly classifies the input so the classifier achieves $100\%$ accuracy.
\end{proof}

\subsection{Invariance to Orthogonal Noise ($\sigma_n$)}
Finally, we generalize \cref{prop:sigma_n_independence} to the multi-class case.
\begin{proposition}
Let $\alpha > 0$ and $\lambda < 1$. In the multi-class setting described in \cref{sec:multiclass}, the generalization accuracy of the optimal classifier is independent of the orthogonal noise scale $\sigma_n$.
\end{proposition}
\begin{proof}
Let the data for class $c$ be $\boldsymbol{x} = \boldsymbol{x}_{\mathrm{sig}} + \sigma_n \boldsymbol{\eta}_{\perp}$, where $\boldsymbol{x}_{\mathrm{sig}} \in \mathcal{S}$ and $\boldsymbol{\eta}_{\perp} \in \mathcal{S}^\perp$.
Decompose the weight matrix as $W = W_{\mathrm{sig}} + W_{\perp}$, where the rows of $W_{\mathrm{sig}}$ lie in $\mathcal{S}$ and rows of $W_{\perp}$ lie in $\mathcal{S}^\perp$.
The logits are:
\begin{equation}\bz = W_{\mathrm{sig}}\boldsymbol{x}_{\mathrm{sig}} + \sigma_n W_{\perp}\boldsymbol{\eta}_{\perp} + \boldsymbol{b}.
\end{equation}
Let $(W, \boldsymbol{b})$ be the minimizer for noise scale $\sigma_n$. The loss depends on the logits $\bz_i$. Consider a scaling of the noise $\sigma_n \to \beta \sigma_n$. The input noise becomes $\beta \sigma_n \boldsymbol{\eta}_{\perp}$.
To maintain the same optimal logit values (and thus the same minimal loss value), the optimizer must scale the orthogonal weight component inversely: $W'_{\perp} = \frac{1}{\beta} W_{\perp}$.
The term relevant for prediction on test data is the inner product of the noise and weights:
\begin{equation}
(\beta \sigma_n \boldsymbol{\eta}_{\perp})^\top (W'_{\perp})^\top = \beta \sigma_n \boldsymbol{\eta}_{\perp}^\top (\frac{1}{\beta} W_{\perp}^\top) = \sigma_n \boldsymbol{\eta}_{\perp}^\top W_{\perp}^\top.
\end{equation}
This term is invariant to $\beta$. Since the signal term $W_{\mathrm{sig}}\boldsymbol{x}_{\mathrm{sig}}$ is unaffected by $\sigma_n$, the distribution of test logits remains identical. Consequently, the accuracy is invariant.

\end{proof}

\section{Supplemental numerical results}
In this section we will present additional numerical results, supporting our results throughout the paper.

\subsection{Discrepancy Between Training and Test Logit Distributions}
\label[appendix]{app:train_test_logit_discrepancy}

Our geometric framework provides a mechanism explaining why test logits cluster less tightly than training logits. In the finite-data regime, the optimizer identifies a weight direction that minimizes the empirical variance of the \textit{training} logits. This is often achieved by exploiting specific noise correlations to "over-cluster" the data --- effectively reducing the empirical variance below the intrinsic population noise level. However, achieving this requires the learned weight vector to deviate from the true signal direction $\boldsymbol{\mu}_f$. This misalignment introduces a non-zero orthogonal component ($\boldsymbol{S}_{\perp}$), which necessarily increases the variance of the test logits (which follow the population statistics). Consequently, the test clusters are looser than those observed during training, as demonstrated in \cref{fig:logit_mean_and_std_vs_sigma_x}.

\begin{figure}[H]
    \centering
    \includegraphics[scale=0.6]{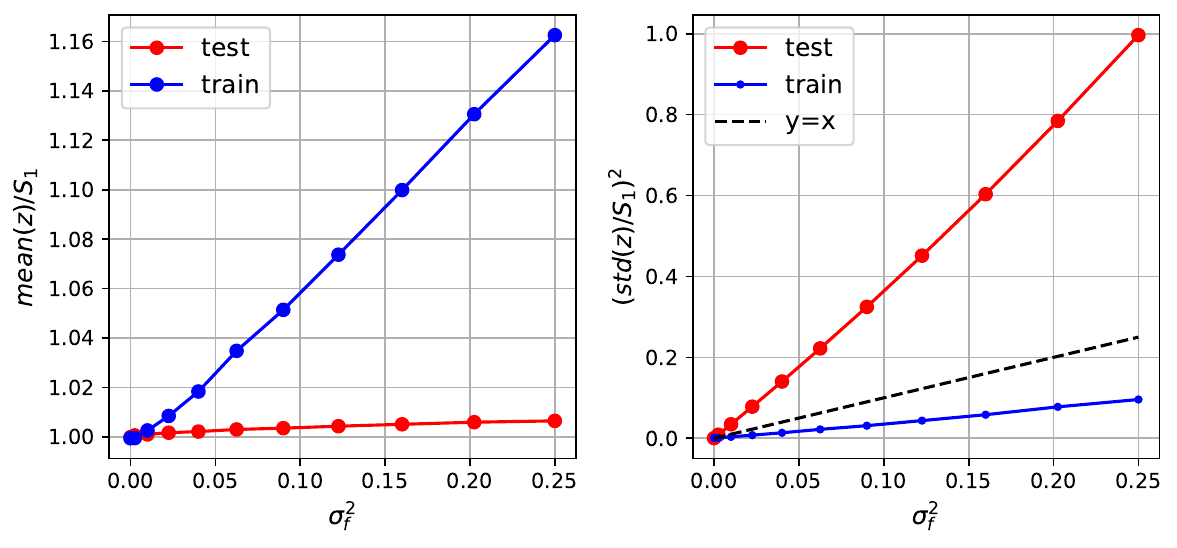}
    \caption{\textbf{Logit statistics vs. signal noise $\sigma_f$.} 
    The mean and standard deviation of the logits in the binary classification setup are shown as functions of $\sigma_{f}$. The diagonal ($y=x$) serves as a reference for the intrinsic signal noise.
    \textbf{Test Set (Generalization):} The variance is given by $\sigma_{z}^{2}=\sigma_{f}^{2}S_{1}^{2}+\sigma_{n}^{2}\|\boldsymbol{S}_{\perp}\|^{2}$. Due to the misalignment term $\|\boldsymbol{S}_{\perp}\|^2$, the test standard deviation typically exceeds the signal contribution, lying above the reference line.
    \textbf{Training Set:} Conversely, the training standard deviation lies below the reference line. This indicates that the optimizer finds a specific direction that minimizes the training logit variance beyond the theoretical lower bound of the signal noise ($\sigma_{z}^{2}S_{1}^{2}$). Here, $\lambda=d/N=0.7$.}
    \label{fig:logit_mean_and_std_vs_sigma_x}
\end{figure}

\subsection{Phase Diagram of Logit Regularization Benefit}

In the main text (\cref{fig:sigma_n_independence_and_phase_diagram}), we presented a phase diagram characterizing the regime where logit regularization improves generalization. That analysis focused on the separable setting ($\lambda=0.7 > 1/2$). Here, we examine the nonseparable regime, $\lambda=0.4 < 1/2$ (see \cref{fig:sigma_n_independence_and_phase_diagram_v2}).

The primary distinction between these regimes lies in the behavior of the unregularized baseline as the orthogonal noise $\sigma_n$ increases.
In the overparameterized case ($\lambda > 1/2$), the training data is typically linearly separable due to the high dimensionality. Consequently, as $\sigma_n$ grows, the unregularized max-margin solution aligns increasingly with the orthogonal noise directions rather than the signal feature. This causes the cosine similarity with the true feature axis to vanish, degrading the generalization accuracy toward random guessing ($1/2$).

In contrast, for $\lambda < 1/2$, the data is generally not linearly separable. Consequently, there is no "easy" overfitting solution that separates the training points using noise dimensions. Instead, increasing the amplitude of the noise will aid alignment: high orthogonal noise effectively drowns out spurious correlations, forcing the optimization to "tune" the margin toward 1 (see bottom panels of \cref{fig:A_test_vs_sigma_n}). 
Therefore, rather than decaying to $1/2$, we empirically observe that the accuracy of the unregularized solution saturates at a non-trivial constant value (see the left panel of \cref{fig:sigma_n_independence_and_phase_diagram_v2}).
This shift in asymptotic behavior alters the topology of the phase diagram. In general, the specific boundaries of the beneficial regime will depend on the geometric ratio $\lambda$ and the underlying data distribution.

\begin{figure}[t!]
    \centering
    \includegraphics[scale=0.6]{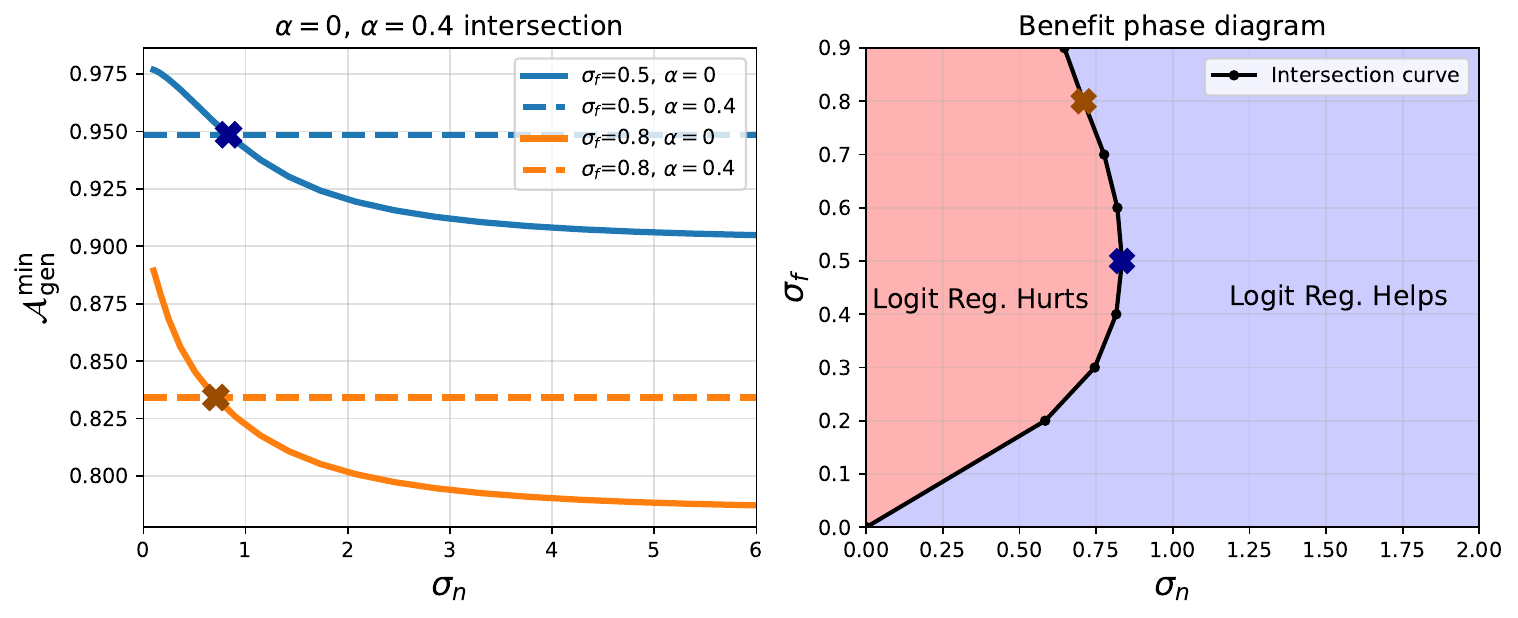}
    \caption{\textbf{Benefit phase diagram for $\lambda<1/2$.}
    \textbf{Left:} Final generalization accuracy as a function of $\sigma_n$, comparing the unregularized ($\alpha=0$) and regularized ($\alpha>0$) cases, and for $\lambda=0.4$. \textbf{Right:} Phase diagram showing regions where logit regularization is beneficial. The boundary curve is constructed from the intersection points (as shown in the left panel) for different values of $\sigma_f$. For illustration, the intersection points of the unregularized and regularized curves for $\sigma_f=0.5, 0.8$ (left) are marked as ``X''s of corresponding colors on the boundary curve (right).}
    \label{fig:sigma_n_independence_and_phase_diagram_v2}
\end{figure}

\subsection{Comprehensive Parameter Sweeps}

In this subsection, we present a systematic numerical analysis of the generalization accuracy and the cosine similarity (relative to the true feature axis). We perform parameter sweeps over the aspect ratio $\lambda$, signal noise $\sigma_f$, and orthogonal noise $\sigma_n$, comparing the unregularized baseline ($\alpha=0$) against the logit-regularized model ($\alpha > 0$).

\begin{figure}[H]
    \centering
    \includegraphics[scale=0.6]{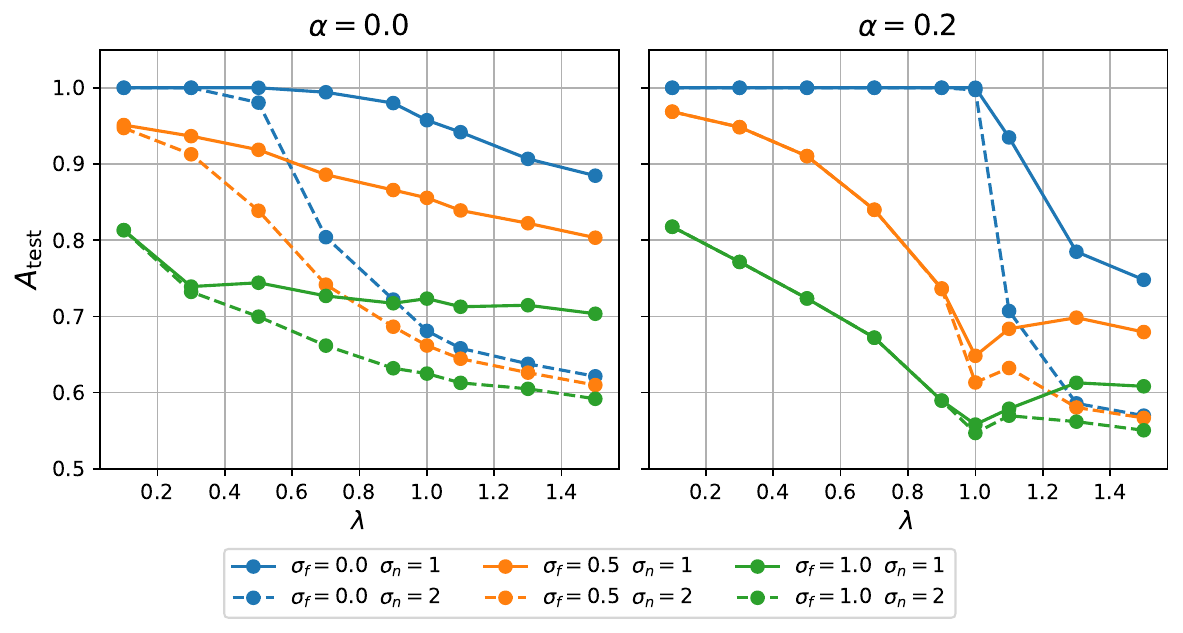}
    \includegraphics[scale=0.6]{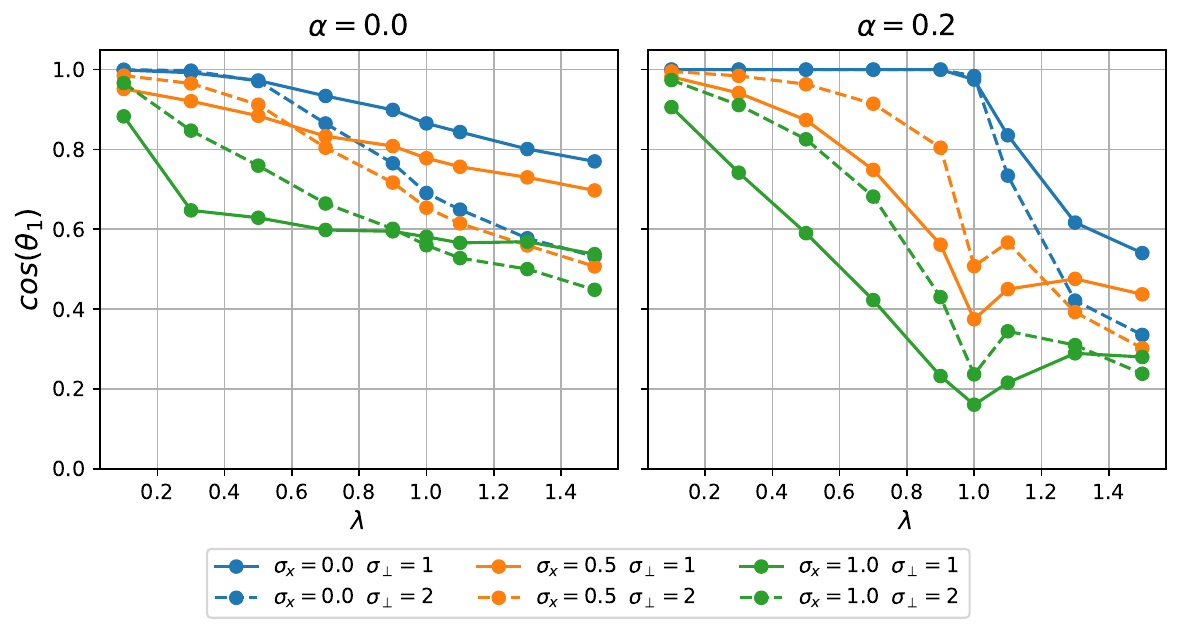}
    \caption{\textbf{Performance as a function of model capacity ($\lambda=d/N$).} 
    \textbf{Top Panels:} Generalization accuracy for the unregularized (left) and regularized (right) models. Different colors correspond to signal noise levels $\sigma_{f}=\{0, 0.5, 1\}$, and solid/dashed lines indicate orthogonal noise levels $\sigma_{n}=\{1, 2\}$.
    Two key phenomena are visible in the regularized case (right): (1) The accuracy becomes independent of $\sigma_{n}$ for $\lambda < 1$, and (2) For the noiseless feature case ($\sigma_f=0$), the interpolation threshold shifts from $\lambda=0.5$ to $\lambda=1$.
    \textbf{Bottom Panels:} Cosine similarity with the true feature axis. Unlike the accuracy, the cosine similarity remains sensitive to $\sigma_n$ in both regimes, confirming that regularization stabilizes the prediction quality even while the weight direction varies.}
    \label{fig:A_test_vs_lambda}
\end{figure}

\begin{figure}[H]
    \centering
    \includegraphics[scale=0.6]{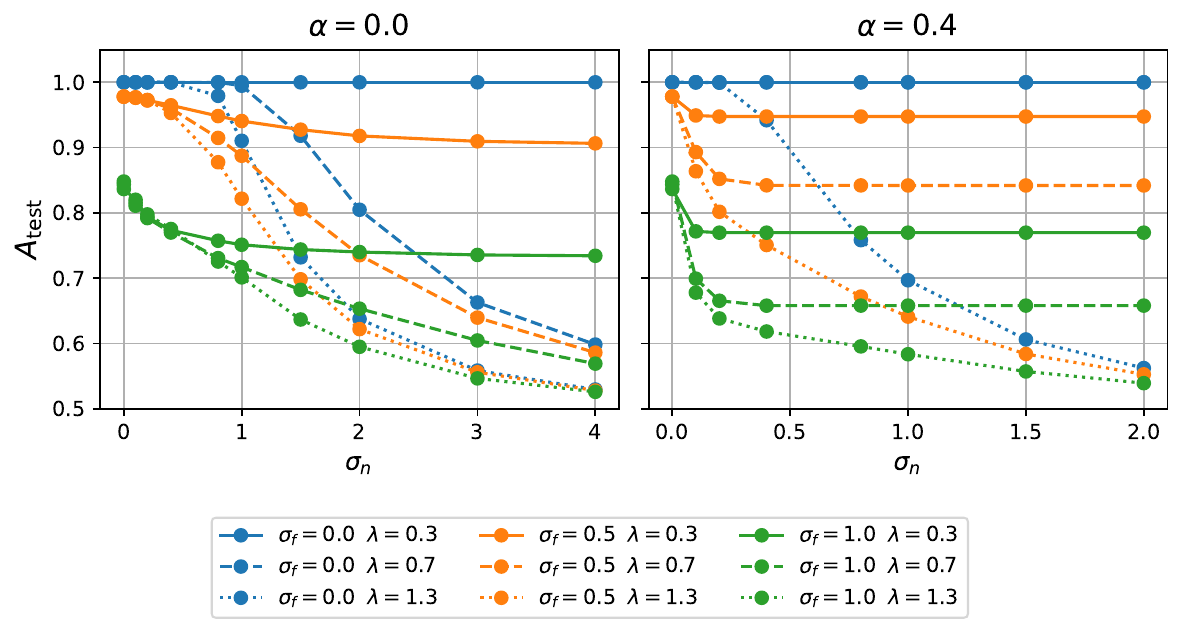}
    \includegraphics[scale=0.6]{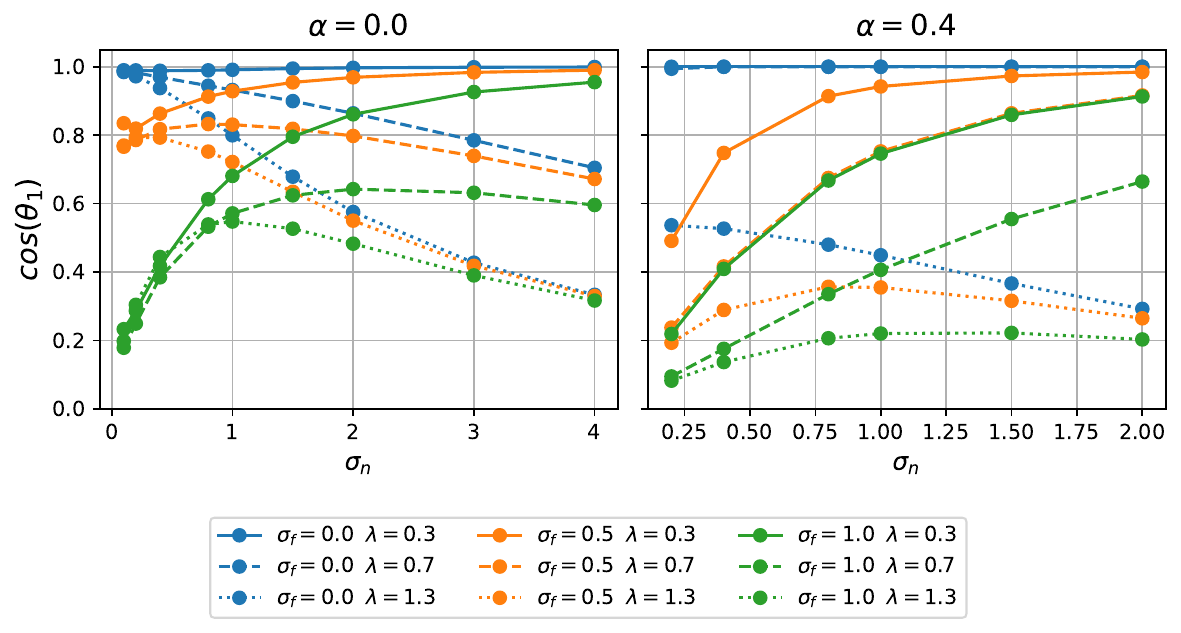}
    \caption{\textbf{Sensitivity to orthogonal noise amplitude $\sigma_n$.}
    \textbf{Top Panels:} Generalization accuracy vs. $\sigma_n$. The regularized model (right) is robust to changes in $\sigma_n$, maintaining constant accuracy. Note: The slight deviation observed at extremely small $\sigma_n$ is a numerical artifact; the convergence time scales inversely with noise, preventing the simulation from fully reaching the theoretical optimum within the fixed computational budget in the limit $\sigma_n \to 0$.
    \textbf{Bottom Panels:} Cosine similarity vs. $\sigma_n$. Here we observe distinct regimes: for the underparameterized case ($\lambda < 1/2$), increasing noise improves alignment (similarity $\to 1$), whereas for the overparameterized case ($\lambda > 1/2$), increasing noise degrades alignment.}
    \label{fig:A_test_vs_sigma_n}
\end{figure}

\subsection{Comparison with Weight Decay}
\label{subsec:comparison_with_weight_decay}

To underscore the unique geometric properties of logit regularization, we contrast its effects with standard $L_2$ parameter regularization (Weight Decay). We consider the objective function:
\begin{equation}
    \mathcal{L}_{\mathrm{WD}} = \sum_{i} \log\left(1+e^{-y_{i}\boldsymbol{S}^{\top}\boldsymbol{x}_{i}}\right) + \frac{\gamma}{2} \|\boldsymbol{S}\|^2.
\end{equation}
While Weight Decay prevents the weights from diverging, it fundamentally differs from logit regularization in how it shapes the optimization landscape. As shown in \cref{fig:A_test_vs_sigma_f_sigma_n_with_WD}, Weight Decay does \textbf{not} induce the specific invariance properties observed under logit regularization.

Specifically, we highlight two key distinctions:
\begin{enumerate}
    \item \textbf{No shift in interpolation threshold:} Unlike logit regularization, adding Weight Decay does not shift the interpolation threshold in the noiseless feature regime ($\sigma_f=0$). The generalization accuracy does not reach $100\%$ throughout the region $\lambda < 1$, confirming that the perfect generalization mechanism (logit clustering around a finite target) is specific to the per-sample convexity of logit penalties.
    \item \textbf{Sensitivity to orthogonal noise:} The final accuracy under Weight Decay remains explicitly dependent on the orthogonal noise scale $\sigma_n$. This stands in sharp contrast to the logit-regularized case, where the accuracy becomes invariant to $\sigma_n$ (as shown in \cref{fig:A_test_vs_lambda}).
\end{enumerate}

It is worth noting that for linear models, quadratic logit regularization ($\alpha z^2$) can be viewed as a generalized Tikhonov regularization of the form $\alpha \boldsymbol{S}^\top \hat{\Sigma} \boldsymbol{S}$, where $\hat{\Sigma} = \frac{1}{N}\sum_i \boldsymbol{x}_i \boldsymbol{x}_i^\top$ is the empirical non-centered covariance matrix. This data-dependent structure naturally adapts the penalty to the geometry of the input features, a property that isotropic Weight Decay ($\gamma \boldsymbol{S}^\top \boldsymbol{I} \boldsymbol{S}$) lacks.

\begin{figure}[H]
    \centering
    \includegraphics[scale=0.6]{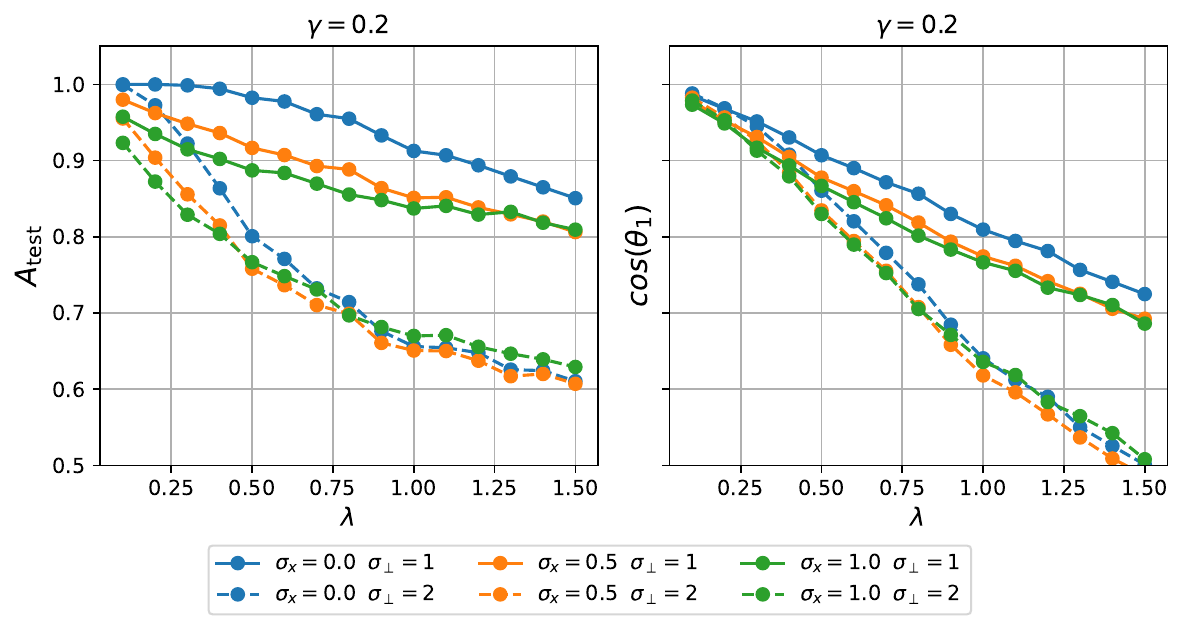}
    \caption{\textbf{Generalization under Weight Decay.} The generalization accuracy (top) and cosine similarity (bottom) achieved at convergence with standard $L_2$ regularization ($\gamma=0.2$). Unlike logit regularization, the accuracy remains sensitive to $\sigma_n$ and does not exhibit perfect generalization in the $\sigma_f=0$ regime for $0.5 < \lambda < 1$.}
    \label{fig:A_test_vs_sigma_f_sigma_n_with_WD}
\end{figure}

\subsection{Dependence of Cosine Similarity on Noise Amplitudes}
\label{subsec:cosine_similarity_dependence}

In this subsection, we numerically investigate the alignment of the learned weights with the true signal direction. We measure this via the cosine similarity $\rho = \cos(\theta_1)$, analyzing its dependence on the orthogonal noise $\sigma_n$ and the signal noise $\sigma_f$.

First, we examine the dependence on $\sigma_n$. Recall from the main text (Eq.~\ref{eq:rho_min}) that the alignment is predicted to follow the scaling form:
\begin{equation}
    \cos(\theta_{1}) = \frac{1}{\sqrt{1+\left(C/\sigma_n\right)^{2}}},
    \label{eq:cosine_scaling_verified}
\end{equation}
where $C$ is a coefficient that depends on $\sigma_f$ but is strictly independent of $\sigma_n$. This implies a linear relationship between $1/\cos^2(\theta_1)$ and $1/\sigma_n^2$. In \cref{fig:cosine_sim_vs_sigma_n}, we verify this prediction numerically: the data perfectly fits the line $y = 1 + C^2 x$, confirming the validity of our analytical derivation.

Next, we characterize the coefficient $C$ by analyzing its dependence on the signal noise $\sigma_f$. We fix $\sigma_n=1$ and track the components of the weight vector $\boldsymbol{S} = S_1 \boldsymbol{e}_1 + \boldsymbol{S}_\perp$. As shown in the left panel of \cref{fig:cosine_sim_vs_sigma_f}, for small $\sigma_f$, the signal component $S_1$ remains approximately constant, while the orthogonal component $\|\boldsymbol{S}_\perp\|$ grows linearly with $\sigma_f$. Consequently, the ratio $C \propto \|\boldsymbol{S}_\perp\|/S_1$ exhibits a linear dependence on $\sigma_f$ in the low-noise regime (see \cref{fig:cosine_sim_vs_sigma_f}, right panel). This demonstrates that as the feature noise vanishes ($\sigma_f \to 0$), the model relies less on spurious orthogonal correlations, recovering the perfect alignment predicted in the noiseless limit.

\begin{figure}[H]
    \centering
    \includegraphics[scale=0.5]{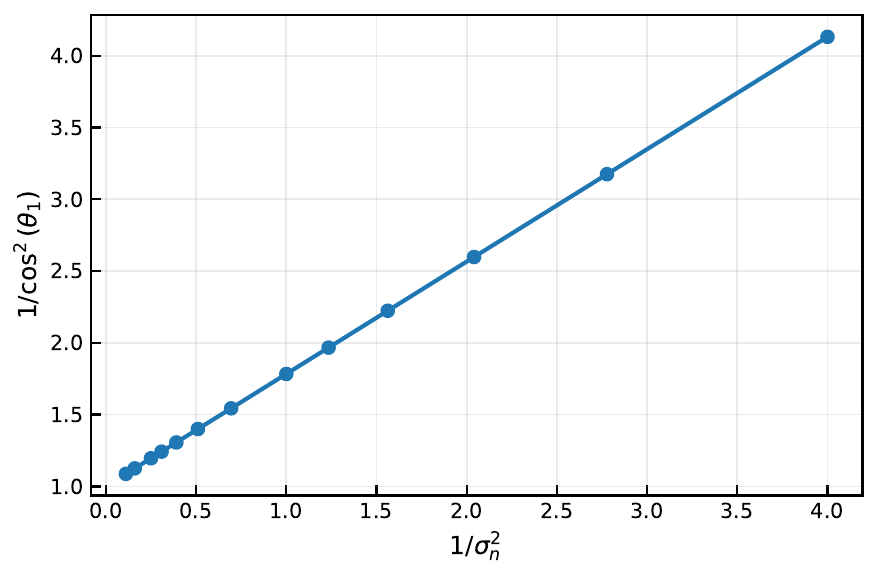}
    \caption{\textbf{Verification of $\sigma_n$ scaling.} Plot of $1/\cos(\theta_{1})^{2}$ versus $1/\sigma_{n}^{2}$. The linear fit confirms the functional form derived in Eq.~(\ref{eq:cosine_scaling_verified}), where the slope corresponds to $C^2$.}
    \label{fig:cosine_sim_vs_sigma_n}
\end{figure}

\begin{figure}[H]
    \centering
    \includegraphics[scale=0.6]{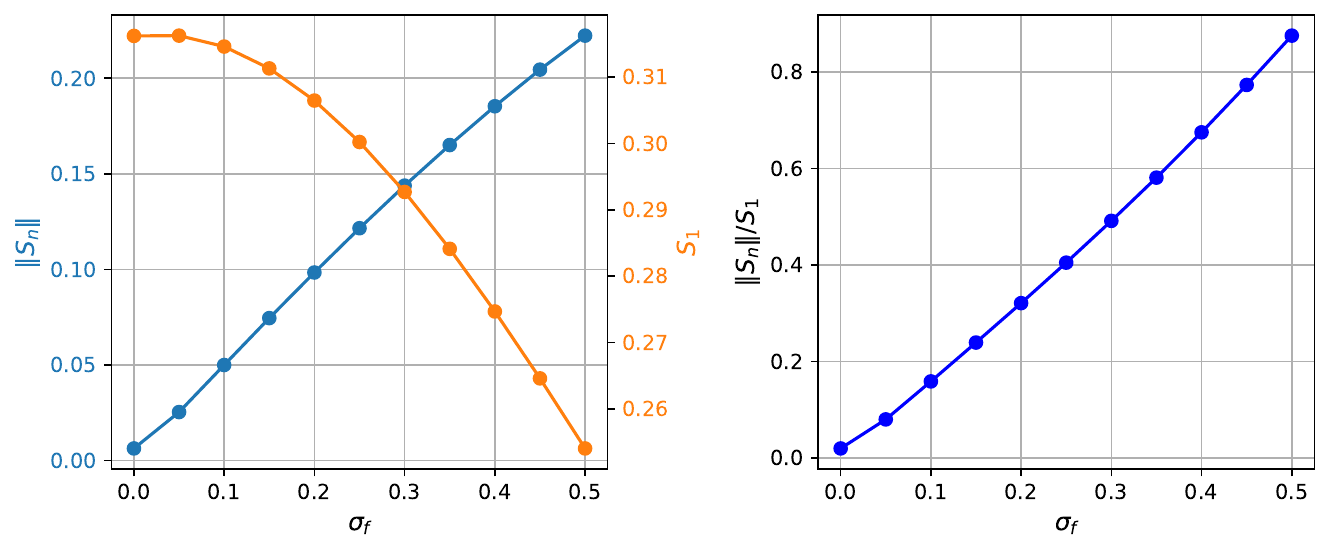}
    \caption{\textbf{Dependence on $\sigma_f$.} All results are for fixed $\sigma_n = 1$.
    \textbf{Left:} The magnitudes of the orthogonal weight component $\|\boldsymbol{S}_{\perp}\|$ and the signal component $S_1$ as functions of $\sigma_f$.
    \textbf{Right:} The ratio $\|\boldsymbol{S}_{\perp}\|/S_{1}$ (which is proportional to $C$) scales linearly with $\sigma_f$ for small noise amplitudes.}
    \label{fig:cosine_sim_vs_sigma_f}
\end{figure}


\subsection{Additional details regarding the numerical settings}
\label[appendix]{app:numerical_details}

Unless stated otherwise, we consider a linear classification model trained by minimizing the regularized loss $\mathcal{L}=\sum_{i} [(1-\alpha)\ell_{\mathrm{CE}}(z_i, y_i) + \alpha z_i^2]$. We use Full-Batch Gradient Descent (GD) for the optimization.
The data is generated according to the signal-plus-noise decomposition $\boldsymbol{x} = y \mu_f \boldsymbol{e}_1 + \boldsymbol{\xi}$, where $\boldsymbol{\xi} = \sigma_f \xi_f \boldsymbol{e}_1 + \sigma_n \boldsymbol{\xi}_{\perp}$. In most experiments, the noise components $\xi_f$ and $\boldsymbol{\xi}_{\perp}$ are drawn from standard Normal distributions.

Below we provide the exact hyperparameters used for each figure:

\begin{itemize}
    \item \textbf{\cref{fig:logit_clustering} (Clustering):}
    \begin{itemize}
        \item \textbf{Top Row ($\alpha=0$):} $d=1400$, $N=2000$ ($\lambda=0.7$), $\mu_f=1$, $\sigma_f=0.5$, $\sigma_n=1$. The model is trained for $20,000$ epochs with learning rate $\eta=0.1$.
        \item \textbf{Bottom Row ($\alpha > 0$):} Same parameters as above, but with $\alpha=0.2$.
        \item \textbf{Figure 1 (Zero Feature Noise):} Same parameters as above, but with $\sigma_f=0$ and $\alpha=0.2$.
    \end{itemize}

    \item \textbf{\cref{fig:alpha_independence} (Invariance with Student-t Noise):} 
    The noise follows a Student-t distribution with degrees of freedom $\nu \in \{2.1, 3, 20\}$. The dimensions are $d=1000$, $N=1500$ ($\lambda \approx 0.67$) with $\mu_f=1$, $\sigma_f=1$, $\sigma_n=1$. 
    For $\alpha=0$, we use GD with $\eta=0.1$. For $\alpha > 0$, we use the Adam optimizer with $\eta=0.001$. The number of epochs varies between $2,000$ and $10,000$ depending on convergence.

    \item \textbf{\cref{fig:alpha_independence_2} (Appendix - Invariance with Gaussian Noise):} 
    Gaussian data with varying dimension $d \in \{200, 1400, 1800\}$ while keeping $N=2000$ fixed, yielding $\lambda \in \{0.1, 0.7, 0.9\}$. Other parameters: $\mu_f=1$, $\sigma_f=0.5$, $\sigma_n=1$. Trained for $20,000$ epochs with GD ($\eta=0.1$).

    \item \textbf{\cref{fig:A_test_vs_lambda} (Performance vs $\lambda$):} 
    We vary $d$ while keeping $N=2000$ fixed to sweep $\lambda$. 
    \begin{itemize}
        \item Parameters: $\mu_f=1$. Signal noise $\sigma_f \in \{0, 0.5, 1\}$. Orthogonal noise $\sigma_n \in \{0.2, 4\}$.
        \item Regularization: We compare $\alpha=0$ vs $\alpha=0.4$.
        \item Optimization: GD with $\eta=0.1$ for $20,000$ epochs.
    \end{itemize}
    
    \item \textbf{\cref{fig:A_test_vs_sigma_f_sigma_n_with_WD} (WD Performance):}
    We vary $d$ while keeping $N=2000$ fixed to sweep $\lambda$,
    \begin{itemize}
        \item Parameters: $\mu_f=1$. $\sigma_f \in \{0, 0.5, 1\}$. $\sigma_n \in \{1, 2\}$.
        \item Regularization: $\alpha=0$ but $\gamma=0.2$ (WD).
        \item Optimization: GD with $\eta=0.1$ for $2,000$ epochs.
    \end{itemize}

    \item \textbf{\cref{fig:cosine_sim_vs_sigma_n} (Cosine Similarity vs $\sigma_n$):} 
    We fix $d=1400$, $N=2000$ ($\lambda=0.7$).
    \begin{itemize}
        \item Parameters: $\mu_f=1$, $\sigma_f=0.5$. We sweep $\sigma_n \in [0.5, 3]$.
        \item Regularization: $\alpha=0.2$.
        \item Optimization: GD with $\eta=0.1$ for $2,000$ epochs.
    \end{itemize}

    \item \textbf{\cref{fig:cosine_sim_vs_sigma_f} (Cosine Similarity vs $\sigma_f$):} 
    We fix $d=1400$, $N=2000$ ($\lambda=0.7$).
    \begin{itemize}
        \item Parameters: $\mu_f=1$, $\sigma_n=1$. We sweep $\sigma_f \in [0, 0.5]$.
        \item Regularization: $\alpha=0.4$.
        \item Optimization: Adam optimizer with $\eta=0.01$ for $1,000$ epochs.
    \end{itemize}

    \item \textbf{\cref{fig:grokking} (Grokking):} 
    We fix $d=1400$, $N=2000$ ($\lambda=0.7$) in the noiseless feature regime.
    \begin{itemize}
        \item Parameters: $\mu_f=0.1$, $\sigma_f=0$, $\sigma_n=1$.
        \item Regularization: We compare $\alpha \in \{0, 0.001, 0.01, 0.1\}$.
        \item Optimization: GD with $\eta=0.1$ for an extended duration of $200,000$ epochs.
    \end{itemize}

    \item \textbf{\cref{fig:sigma_n_independence_and_phase_diagram} (Phase Diagram $\lambda=0.7$):} 
    $d=1400$, $N=2000$. We interpolate theoretical curves based on simulations with $\sigma_n$ sweeping from $0.1$ to $6$.
    \begin{itemize}
        \item Parameters: $\mu_f=1$. $\alpha=0.4$.
        \item Optimization: GD with $\eta=0.1$ for $20,000$ epochs.
    \end{itemize}
    
    \item \textbf{\cref{fig:sigma_n_independence_and_phase_diagram_v2} (Phase Diagram $\lambda=0.4$):} 
    $d=800$, $N=2000$.
    \begin{itemize}
        \item Parameters: $\mu_f=1$. $\alpha=0.4$.
        \item Optimization: GD with $\eta=0.1$ for $20,000$ epochs.
    \end{itemize}

    \item \textbf{\cref{fig:multiclass_clustering} (Multi-class):} 
    $K=10$ classes. $d=1400$, $N=2000$ ($\lambda=0.7$).
    \begin{itemize}
        \item Parameters: $\mu_f=1$, $\sigma_f=0.1$, $\sigma_n=1$.
        \item Regularization: $\alpha=0$ and $\alpha=0.2$.
        \item Optimization: GD optimizer with $\eta=0.1$ for $20,000$ epochs.
    \end{itemize}

    \item \textbf{\cref{fig:multiclass_grokking} (Multi-class Grokking):} 
    $K=10$ classes. $d=1400$, $N=2000$ ($\lambda=0.7$). Same as previous but for $\sigma_f=0$ and $\alpha=0.04$.
\end{itemize}

\subsubsection{ResNet-18 (\cref{fig:feature_geometry_motivation,fig:Penultimate_resnet18_results})}

We extract 512-dimensional penultimate features from a frozen ResNet-18 pre-trained on ImageNet. For the linear probe experiments, we select $N_{\text{train}}=1000$ balanced training samples and use the remaining $N_{\text{test}}=9000$ samples for evaluation. We compare performance on the clean CIFAR-10 test set and on CIFAR-10-C (Impulse Noise, Severity 5).

\textbf{Feature Decomposition and Noise Estimation:}
To map the complex feature distribution to our theoretical signal-plus-noise model, we perform a class-conditional decomposition.
Let $X \in \mathbb{R}^{N \times d}$ be the matrix of feature vectors and $y$ the labels. For each class $c$, we compute the empirical mean $\boldsymbol{\mu}_c \in \mathbb{R}^d$ and the centered covariance matrix $\Sigma_c$.
\begin{enumerate}
    \item \textbf{Signal Subspace:} We define the signal subspace $\mathcal{S}$ as the span of the differences between class means, $\mathcal{S} = \text{span}(\{\boldsymbol{\mu}_c - \boldsymbol{\mu}_1\}_{c=2}^K)$. We compute an orthonormal basis $Q_1 \in \mathbb{R}^{d \times (K-1)}$ for $\mathcal{S}$ via QR decomposition.
    
    \item \textbf{Orthogonal Subspace:} We complete the basis to $\mathbb{R}^d$ using an orthogonal complement $Q_2 \in \mathbb{R}^{d \times (d-K+1)}$, such that $U = [Q_1, Q_2]$ is unitary.
    
    \item \textbf{Noise Projection and RMS Calculation:} We project the class-conditional covariances onto these subspaces. The effective noise amplitudes reported in the figures are calculated as the root-mean-square (RMS) of the standard deviations in each subspace, averaged over all classes:
    \begin{equation}
        \sigma_f^{\text{eff}} = \frac{1}{K} \sum_{c=1}^K \sqrt{\text{Tr}(Q_1^\top \Sigma_c Q_1) / \text{dim}(\mathcal{S})}, \quad
        \sigma_n^{\text{eff}} = \frac{1}{K} \sum_{c=1}^K \sqrt{\text{Tr}(Q_2^\top \Sigma_c Q_2) / \text{dim}(\mathcal{S}^\perp)}.
    \end{equation}
    To allow for comparison across datasets with different signal scales, we normalize these values by the mean pairwise distance between class centroids.
\end{enumerate}

\textbf{Synthetic Orthogonal Scaling:}
To verify the invariance to orthogonal noise (Fig. 11, Right), we synthetically manipulate the data geometry. We project every feature vector $\boldsymbol{x}_i$ onto the orthogonal subspace via the projection matrix $P_{\perp} = Q_2 Q_2^\top$ and scale this component by a factor $\gamma$:
\begin{equation}
    \boldsymbol{x}_i^{\text{scaled}} = (I - P_{\perp})\boldsymbol{x}_i + \gamma P_{\perp}\boldsymbol{x}_i.
\end{equation}
This operation scales the effective $\sigma_n$ by $\gamma$ while preserving the signal geometry exactly. The linear classifier is then trained on these scaled features using the Adam optimizer ($\eta=10^{-4}$) for $30,000$ epochs.

\end{document}